\documentclass[journal]{IEEEtran}
\usepackage[T1]{fontenc}

\usepackage{cite}
\usepackage[pdftex]{graphicx}
\DeclareGraphicsExtensions{.pdf,.jpeg,.png}
\usepackage{amsmath}
\usepackage[noend]{algpseudocode} 
\usepackage{array}
\ifCLASSOPTIONcompsoc
  \usepackage[caption=false,font=normalsize,labelfont=sf,textfont=sf]{subfig}
\else
  \usepackage[caption=false,font=footnotesize]{subfig}
\fi
\usepackage{algorithm} 
\usepackage{amssymb}
\usepackage{amsthm}
\usepackage{pgfplots}

\pgfplotsset{compat=1.13}

\newtheorem{theorem}{Theorem}
\newtheorem{definition}[theorem]{Definition}
\newtheorem{proposition}[theorem]{Proposition}

\makeatletter
\newcommand{\StatexIndent}[1][3]{%
  \setlength\@tempdima{\algorithmicindent}%
  \Statex\hskip\dimexpr#1\@tempdima\relax}
\algdef{S}[WHILE]{WhileNoDo}[1]{\algorithmicwhile\ #1}%
\makeatother

\DeclareMathOperator*{\argmin}{arg\,min}
\DeclareMathOperator*{\argmax}{arg\,max}

\pgfplotsset{
	colormap={blueyellow}{rgb255(0cm)=(0,0,255); rgb255(1cm)=(255,255,0)}
}

\definecolor{C1}{RGB}{215,25,28}
\definecolor{C2}{RGB}{171,221,164}
\definecolor{C3}{RGB}{43,131,186}
\definecolor{C4}{RGB}{253,174,97}

\begin{document}
\bstctlcite{IEEEexample:BSTcontrol}

\title{Multi-robot persistent surveillance with connectivity constraints}



\author{J\"urgen~Scherer,
        Bernhard~Rinner$^{1}$%
\thanks{$^{1}$Both authors are with the Institute of Networked and Embedded Systems, University of Klagenfurt, Austria,
{\tt\footnotesize \{juergen.scherer, bernhard.rinner\}@aau.at} }
}


\maketitle

\begin{abstract}
Mobile robots, especially unmanned aerial vehicles (UAVs), are of increasing interest for surveillance and disaster response scenarios. We consider the problem of multi-robot persistent surveillance with connectivity constraints where robots have to visit sensing locations periodically and maintain a multi-hop connection to a base station. We formally define several problem instances closely related to multi-robot persistent surveillance with connectivity constraints, i.e.,  connectivity-constrained multi-robot persistent surveillance (CMPS), connectivity-constrained multi-robot reachability (CMR), and connectivity-constrained multi-robot reachability with relay dropping (CMRD), and show that they are all NP-hard on general graph. We introduce three heuristics with different planning horizons for convex grid graphs and combine these with a tree traversal approach which can be applied to a partitioning of non-convex grid graphs (CMPS with tree traversal, CMPSTT). In simulation studies we show that a short horizon greedy approach, which requires parameters to be optimized beforehand, can outperform a full horizon approach,  which requires a tour through all sensing locations, if the number of robots is larger than the minimum number of robots required to reach all sensing locations. The minimum number required is the number of robots necessary for building a chain to the farthest sensing location from the base station. Furthermore, we show that partitioning the area and applying the tree traversal approach can achieve a performance similar to the unpartitioned case up to a certain number of robots but requires less optimization time.
\end{abstract}


\IEEEpeerreviewmaketitle

\section{Introduction}
\IEEEPARstart{M}{obile} robots, especially unmanned aerial vehicles (UAVs), are of increasing interest in various application domains including surveillance. Examples for the deployment of robots for surveillance scenarios are disaster response~\cite{Erdelj2017}, \cite{Scherer2015}, \cite{Khan2018}, wildfire monitoring~\cite{Ghamry2016}, security tasks~\cite{Liu2013}, environmental monitoring~\cite{Rossi2016}, and exploration and mapping \cite{Masehian2017}, \cite{Santos2017}. Persistent surveillance is the task of continuously monitoring an environment over a longer period of time. The potentially large area and the limited sensor view of robots requires a movement strategy such that every point of interest in the area gets visited periodically by the robots. In disaster response scenarios, it is also crucial that the mission operators are aware of the situation at any time during the mission, which implies that the robots have to continuously report the state of the mission and the captured data to the base station. Wireless transceivers enable the robots to exchange data over a limited distance only, and therefore it is necessary to transmit the data over multiple hops if a continuous connectivity to the base station is required and the area is larger than the communication range. A continuously connected network of robots and the base station further allows to track the state of the (aerial) robots for safety reasons.

In this paper we investigate in the problem of persistent surveillance with continuous connectivity constraints. Given a representation of the area, the number of robots, the positions of the points of interest (which we denote as sensing locations) and the base station, the problem is to find a path for each robot that minimizes the worst idleness for all sensing locations such that the network of robots and base station is connected throughout the mission. The idleness of a sensing location at a certain instant is defined as time that passed since the last visit by any robot, and the worst idleness is the maximum idleness over all sensing locations and over the whole mission duration. Since the duration of a persistence surveillance mission is potentially infinite, strategies are necessary that generate solutions for an infinite time horizon. Due to the connectivity constraint, the robots mutually restrict their possible movements and traditional patrolling strategies (e.g. \cite{Portugal2017}, \cite{Pasqualetti2012b}, \cite{Lauri2014}) cannot be directly applied because they do not coordinate the robots' movement in space and time for continuous connectivity. Online multi-robot persistent surveillance algorithms often need the state of the whole environment to make a decision about the next action. Therefore, they rely implicitly on some communication mechanism allowing the robots to exchange information, but do not consider a limited communication range \cite{Nigam2012}, \cite{Franco2013}.

Path planning for mobile robots is often based on an abstract representation of the environment obtained by some discretization technique. In this work we consider graphs \cite{Banfi2018} and grids \cite{Nigam2012}. In a graph representation two types of edges describe whether a robot can move between two vertices and whether robots can communicate when placed at two different vertices at the same time. A grid represents a special type of graph where a robot can move between the neighboring cells. We study different problem instances related to connectivity-constrained multi-robot persistent surveillance. In particular, we define the problems connectivity-constrained multi-robot persistent surveillance (CMPS), connectivity-constrained multi-robot reachability (CMR), and connectivity-constrained multi-robot reachability with relay dropping (CMRD). The latter two problems are concerned with reachability of vertices when all robots start at a dedicated base station vertex. We show that all these problems are NP-hard on graphs.

Figure~\ref{fig:scenario_cmpstt} illustrates the different problems in a grid-based scenario composed by four convex partition around an obstacle. Sensing robots aim to cover the partitions while maintaining connectivity to the base station with the help of relay and release robots. The relay and release robots do not move while a partition is covered. Release robots are robots that stay at a release point. The release point is the starting position where the sensing robots start to cover the partition after they have moved together with the release robot to the release point. CMPS considers the movement planning within one convex area, CMR deals with the question which vertices/cells are reachable from the base station, and CMRD deals with the problem of placing relay robots. CMPSTT is concerned about the order of covering the partitions and the number of robots assigned to partitions. We assume that robots can change roles and can take on every role.

Because of the complexity of finding feasible solutions, we do not attempt to solve CMPS on general graphs. In previous works \cite{Scherer2016}, \cite{Scherer2017} we investigated in different strategies on grids where cells within a certain distance are within communication range. In contrast to our previous work we do not consider energy constraints in this work. These strategies can only be applied on convex grids without holes, which arise in the presence of obstacles. A tree traversal algorithm \cite{Mosteo2009} with relay dropping can be applied to more general environments but does not perform as well as the suggested strategies in convex scenarios with many sensing locations. Therefore, we suggest a combination of tree traversal and coverage (CMPS with tree traversal, CMPSTT), which can be applied after a partitioning of an arbitrary shaped environment into convex partitions and ensures the connectivity constraint. We adapt the algorithm of \cite{Mosteo2009} and show that determining the optimal order for visiting the partitions of a given partitioning to minimize the worst idleness is also NP-hard. Note that we have studied related patrolling problems of minimizing or constraining the delay between data generation at the robots and data arrival at the base station with \textit{relaxed} connectivity constraints in \cite{Scherer2019_MDTD} and \cite{Scherer2019_MILC}.

\begin{figure}[t]
	\centering
	\includegraphics[scale=0.4]{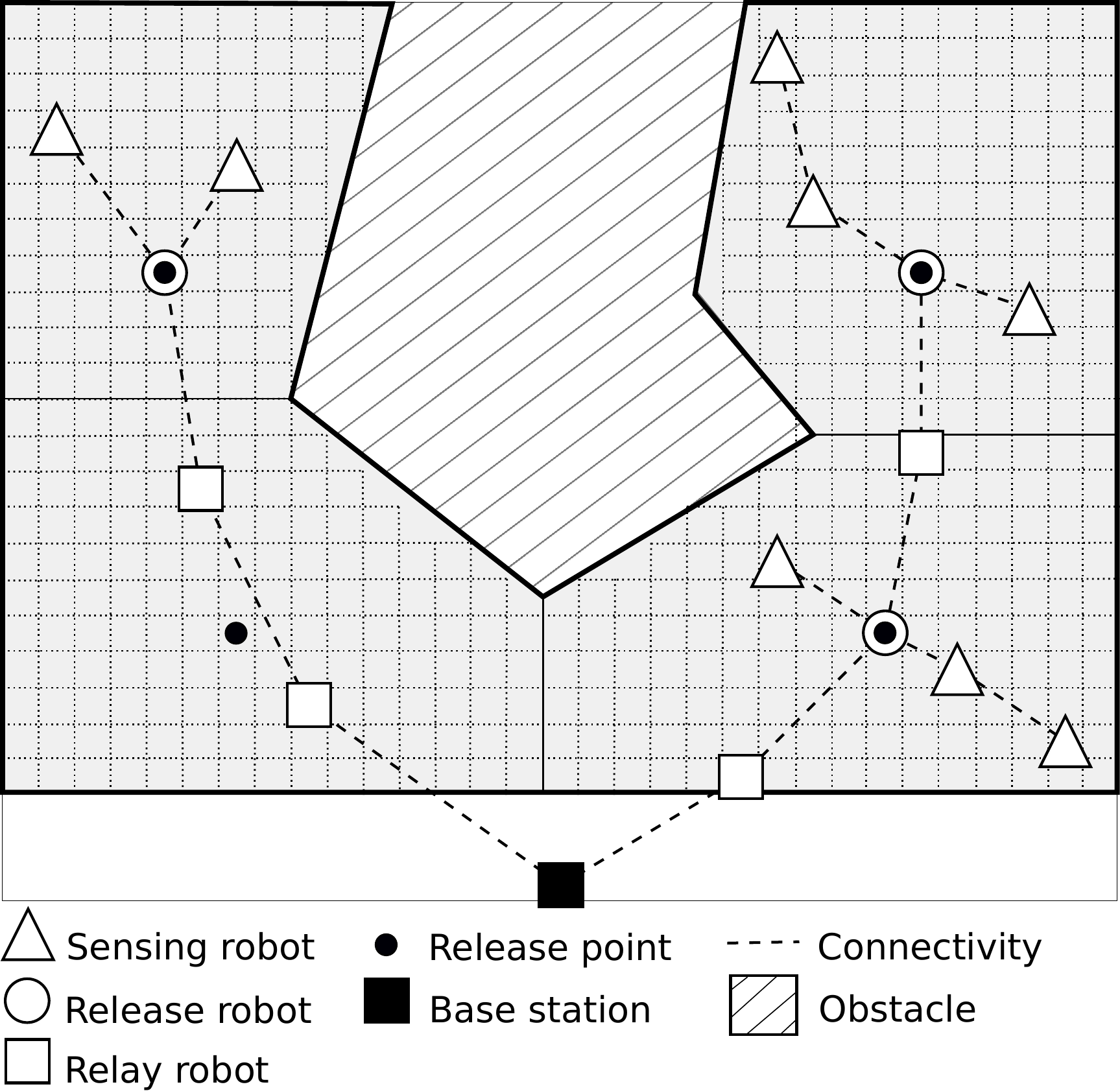}
	\caption{Problem instances of multi-robot persistent surveillance with connectivity constraints depicted on a grid environment with four convex partitions around an obstacle. The lower left partition is currently not covered in the depicted scenario. In a potential solution for CMPSTT the robots in the upper left partition gather at the release robot and retreat together with the relay robot to the release point of the lower left partition to cover this partition.}
	\label{fig:scenario_cmpstt}
\end{figure}

The contributions of this work can be summarized as follows: (i) We define CMPS, CMR, CMRD on general graphs and prove that they are NP-hard, (ii) we propose strategies for convex grid areas, (iii) we combine tree traversal and area coverage for partitioned environments and show that CMPSTT is NP-hard, and (iv) we perform an extensive simulation study to assess the performance of the proposed strategies.

The remainder of the article is organized as follows: Section~\ref{sec:relatedwork} discusses related work, Section~\ref{sec:problem} introduces the notation and investigates in the complexity on general graphs, Section~\ref{sec:convex} describes the strategies on convex grid areas, Section~\ref{sec:partition} describes an extension to partitioned areas, Section~\ref{sec:eval} presents the results of the simulation studies, and Section~\ref{sec:conclusion} concludes the paper.

\section{Related work}
\label{sec:relatedwork}

Multi-robot persistent surveillance is related to the multi-robot patrolling problem. This problem is usually concerned with determining closed paths, which are continuously traversed by robots, or with controlling the robots along predefined paths to optimize for some performance metric \cite {Acevedo2013_ICUAS}, \cite{Smith2012}, \cite{Pasqualetti2012a}, \cite{Pasqualetti2012b}. Tour planning is often tackled from the operations research perspective by solving a vehicle routing problem or a multi-traveling salesperson problem on a graph. In \cite{Mitchell2015} tours for multiple agents are planned such that the visit frequency of targets is maximized and agents can refuel at refuel depots. Mersheeva and Friedrich \cite{Mersheeva2015} perform path planning for multiple UAVs that repeatedly visit targets with different priorities. UAVs change batteries at base stations, and the planning horizon is determined by the available number of batteries. Manyam et al. \cite{Manyam2017} consider also the data delivery latency to the base station. Keller et al. \cite{Keller2017} present a tour planning approach taking into account turning radius constraints of gliders. In \cite{Lan2016} trajectories are planned to estimate a spatio-temporal field in a dynamic environment. A key difference to our work is that  continuous connectivity of the robots to a base station has not been considered at all.

Nigam et al. \cite{Nigam2012} introduce a control strategy where each UAV selects the next cell to visit based on the idleness of the cell, the distance between the UAV and the cell, and the distances between the other UAVs and the cell. The infinite horizon persistence surveillance problem is converted to a short horizon problem where the necessary parameters are determined by an offline optimization approach. Franco et al. \cite{Franco2013} present a controller adjusting speed and direction to reach a desired long term coverage profile. Santos et al. \cite{Santos2017} investigate in exploring a spatio-temporal field with different goal selection strategies for a single robot. If the motion planning is done online with the approaches in \cite{Nigam2012}, \cite{Franco2013}, the robots need to know the state of the whole environment and of all other robots. This knowledge requires communication between the robots or the base station and the robots. However, the discussed related work does not consider limited communication range.

Maintaining connectivity while executing spatially distributed tasks is a recurrent requirement in mobile-robot systems. Panterati et al. \cite{Panerati2018} present a swarm of robots that move in compact formation through a convex environment and stops at certain locations to visit multiple task locations simultaneously by building one chain of robots towards each task location. The order in which the tasks should be visited is determined by solving a mixed integer linear program (MILP). A tree growing algorithm \cite{Majcherczyk2018} arranges the robots into a star-like topology to visit multiple task locations. Ponda et al. \cite{Ponda2012} present an algorithm for allocation of data streaming and relaying tasks to maintain connectivity to a base station during task execution. Gr{\o}tli and Johansen \cite{Grotli2012} investigate in offline mission planning for multiple UAVs to visit targets and stream data to a base station by solving a MILP model and exploiting a radio propagation path loss simulator to take bandwidth requirements into account. Flushing et al. \cite{Flushing2017} present also a MILP model for joint task scheduling and data routing and transmission scheduling. Finally, Zavlanos et al. \cite{Zavlanos2013} present a distributed algorithm for maintaining communication requirements with fixed infrastructure. All these approaches focus on connectivity maintenance with or without explicit task allocation but without explicitly considering persistence visits of a given set of tasks. Solving MILPs for path planning problems is demanding due to the computational complexity of integer programming.

The computational complexity of planning the robots' movement on graphs has been investigated in literature. Hollinger and Singh \cite{Hollinger2012} investigate in the multi-robot path planning with connectivity constraints on graphs. They show that the problem of planning paths for multiple robots from start to goal positions while maintaining connectivity in every time step is NP-hard, but some details about how to construct the graph are missing. Tateo et al. \cite{Tateo2018} show that the related problem, which they call multi-agent connected path planning problem (MCPP), is PSPACE-complete. Anisi et al. \cite{Anisi2010} show that the related problem of planning a path in a visibility graph such that several targets can be observed is NP-hard. Banfi et al. \cite{Banfi2018} show that the problem of planning paths to reach a connected configuration on a graph in minimum time is NP-hard. In contrast to all these investigations we focus on persistent surveillance with the connectivity requirement to a base station.

As a summary, our work combines persistent surveillance with connectivity requirements and appropriate motion planning. While there is a lot of work on each of these areas, we are not aware of related work that can be directly applied to our problem. The considered exploration task in \cite{Mosteo2009} represents the most relevant work on connectivity-constrained motion planning with goal selection, and we use it as a baseline for our simulation study. We employ the work in \cite{Nigam2012} on goal selection for persistent surveillance, and a comparison with other approaches for persistent surveillance has been provided in \cite{Nigam2012}. Work that approach connectivity-constrained motion from a control-theoretic perspective (e.g. \cite{Zavlanos2013}, \cite{Zavlanos2008}, \cite{Schuresko2012}) do often not consider a particular application and goal selection. As a consequence, goal selection and scheduling of the robots for persistent surveillance has to be incorporated into these approaches.

\section{Problem formulation and complexity analysis}
\label{sec:problem}

\begin{table}[t]
{\footnotesize
\begin{tabbing}
	\textbf{Symbol} \hspace{6em} \= \textbf{Meaning} \\
	$G_M=(V, E_M)$ \> (undirected) movement graph with \\
	\> vertex set $V$ and edge set $E_M$ \\
	$G_C=(V, E_C)$ \> (undirected) connectivity graph with \\
	\> vertex set $V$ and edge set $E_C$ \\
	$G_C\langle W\rangle$ \> vertex induced subgraph of $G_C$ \\
	\> (defined by vertex set $W$) \\
	$b$ \> base station vertex, or base station partition \\
	$[v,w] \in E$ \> (undirected) edge between $v$ and $w$ \\
	$V_S$ \> set of sensing location vertices \\
	$n$ \> number of vertices in $V$ \\
	$R=\{1, \ldots, r\}$ \> set of $r$ robots \\
	$p_t(i)$ \> position of robot $i$ at time $t$ \\
	$p_t$ \> vector of positions of all robots at time $t$ \\
	$\pi$ \> sequence of all robot positions \\
	$G_C^t$ \> vertex induced subgraph of $G_C$ at time $t$ \\
	\> (defined by set of vertices in $p_t$ and $b$) \\
	$I_t^{\pi}(v)$ \> instantaneous idleness of vertex $v$ at time $t$ \\
	\> (using $\pi$) \\
	$WI_t^{\pi}(v)$ \> instantaneous worst idleness of vertex $v$ at time $t$ \\
	\> (using $\pi$) \\
	$WI$ \> worst idleness \\
	$CT$ \> coverage time \\
	$dist_{G}(s,d)$ \> length of shortest path between vertices \\
	\> $s$ and $d$ in graph $G$ \\
	$\mathbb{R}_{\geq 0}, \mathbb{R}_{> 0}$ \> set of real numbers greater than or equal 0, and \\
	\> greater than $0$, respectively \\
	CMPS \> connectivity-constrained multi-robot persistent \\
	\> surveillance \\
	CMPSTT \> CMPS with tree traversal \\
	CMR \> connectivity-constrained multi-robot reachability \\
	CMRD \> CMR with relay dropping \\
	FH \> full horizon algorithm \\
	SH \> short horizon algorithm \\
	SHC \> short horizon cooperative algorithm \\
	TT \> tree traversal algorithm \\
\end{tabbing}
}
\caption{List of symbols and abbreviations.}
\label{tab:symbols}
\end{table}

This section introduces the notation and provides formal definitions of the problems. The used symbols and abbreviations are summarized in Table~\ref{tab:symbols}. A set $R=\{1, \ldots, r\}$ of robots is available for the surveillance of an environment, which is modeled with two graphs: the movement graph $G_M=(V, E_M)$ and the connectivity graph $G_C=(V, E_C)$. Both graphs share the same set of vertices $V$ (with $|V|=n$), which describe possible positions of robots at discrete points in time. Time is divided into time steps, and the positions of the robots at time step $t$ is denoted $p_t=(p_t(1), \ldots, p_t(r))$, with $p_t(i) \in V$. A subset $V_S \subseteq V$ represents sensing locations. When there is an edge $[v,w] \in E_M$, a robot can move from $v$ to $w$ within one time step. An edge $[v,w] \in E_C$ means that two robots, or a robot and the base station $b \in V$, are able to transfer data to each other when they are at positions $v$ and $w$ at the same time. We call the tuple $(G_M,G_C,b,V_S)$ describing an environment simply ``graph'' and refer explicitly to the graphs $G_M$ and $G_C$ if necessary.

A patrolling solution $\pi$ is a mapping from instants of time to vertices in $V$ for every robot and describes when vertices are visited by the robots. At each time step an \emph{instantaneous idleness} is associated with a vertex $v \in V_S$. This value describes the time passed since the last visit of the vertex. The definition of the idleness criterion adheres to the definition in \cite{Lauri2014}:

\begin{definition}[Instantaneous idleness, instantaneous worst idleness, worst idleness criterion \cite{Lauri2014}]
If the robots follow a solution $\pi$, the instantaneous idleness $I_t^{\pi}(v) \in \mathbb{R}_{\geq 0}$ at time $t$ of vertex $v \in V_S$ is the elapsed duration since the last visit of $v$ by any robot. By convention, at initial time, $I_0^{\pi}(v) = 0$, for any solution $\pi$ and each $v \in V_S$. The worst idleness criterion $WI^{\pi}$ is defined as 
\begin{equation}
	WI^{\pi}:=\limsup_{t \rightarrow +\infty}{WI_t^{\pi}}
\end{equation}
where $WI_t^{\pi}$ is the instantaneous worst idleness and is defined as $WI_t^{\pi} := \max_{v \in V_S}{I_t^{\pi}(v)}$.
\end{definition}

A solution of a persistence surveillance problem can be described as a sequence of positions $\pi:=(p_0, p_1, \ldots)$ with each $p_{t+1}$ resulting from $p_t$ where each robot moves along an edge of $E_M$ (or stays at the same vertex). A valid solution for CMPS has the property that at each time step $t$ the vertex induced subgraph $G_C^t:={G_C\langle\{p_t(i): i \in R\} \cup \{b\}\rangle}$ is connected, i.e. all robots are connected with the base station. It is allowed that multiple robots can move to the same vertex at the same time.

To analyze the complexity of the problem of finding a minimum $WI$ solution, we define CMPS as decision problem (the problem of deciding whether a given tuple belongs to a set of tuples):
\begin{definition}[d-CMPS]
Problem \mbox{d-CMPS} is a set of tuples of the form $(G_M,G_C,b,V_S,p_0,1^T)$ where $p_0 \in V^{r}$ are the initial positions of the robots and $T$ is a time bound. The elements of the set have the properties: (i)~$p_0(i)=v \in V_S$ for some robot $i$, (ii)~there is a sequence $(p_0, \ldots, p_{t'}), t' \leq T$ such that $p_{t'}=p_0$, each $v \in V_S$ is visited by some robot (i.e. there is a robot $i$ with $p_t(i)=v$ for $0 \leq t \leq t'$), and (iii)~the vertex induced subgraph $G_C^t$ is connected for $0 \leq t \leq t'$, and (iv)~${r<n}$.
\end{definition}

In this definition $1^T$ is a string of 1s of length $T$. The reason for this definition is that we can provide a polynomial transformation from 3SAT\footnote{A 3SAT instance consists of a set $W=\{x_1, \ldots, x_\alpha\}$ of Boolean variables, and a set $C=\{c_1, \ldots, c_\beta\}$ of clauses where each clause contains exactly three literals. The literals are of the form $x_i$ or $\overline{x}_i$ where $x_i \in W$. The question is, whether there is an assignment of values from $\{True, False\}$ to the variables such that in every clause at least one literal evaluates to $True$.} to \mbox{d-CMPS} to show that the problem is NP-hard (actually it is NP-complete). Using the definition $(G_M,G_C,b,V_S,p_0,T)$, a transformation would have to check a sequence of at most $T$ positions to derive a solution for the 3SAT instance. Such a check would result in an exponential time transformation. Additionally, to verify whether a solution is a valid solution for an \mbox{d-CMPS} instance, it would be necessary to iterate over exponential many positions, and therefore, the problem would also not be in $NP$.

\subsection{NP-hardness results}

To show the NP-hardness of \mbox{d-CMPS}, we provide a transformation from a 3SAT instance to a \mbox{d-CMPS} instance. The transformation for the 3SAT instance $\{c_1=\{x_1, x_2, x_3\}, c_2=\{\overline{x}_1, \overline{x}_2, x_4\}, c_3=\{x_2, \overline{x}_3, \overline{x}_4\}\}$ is shown in Figure~\ref{fig:cmps_3sat}. The general transformation is described in the proof of

\begin{figure}[t]
	\centering
	\includegraphics[scale=0.4]{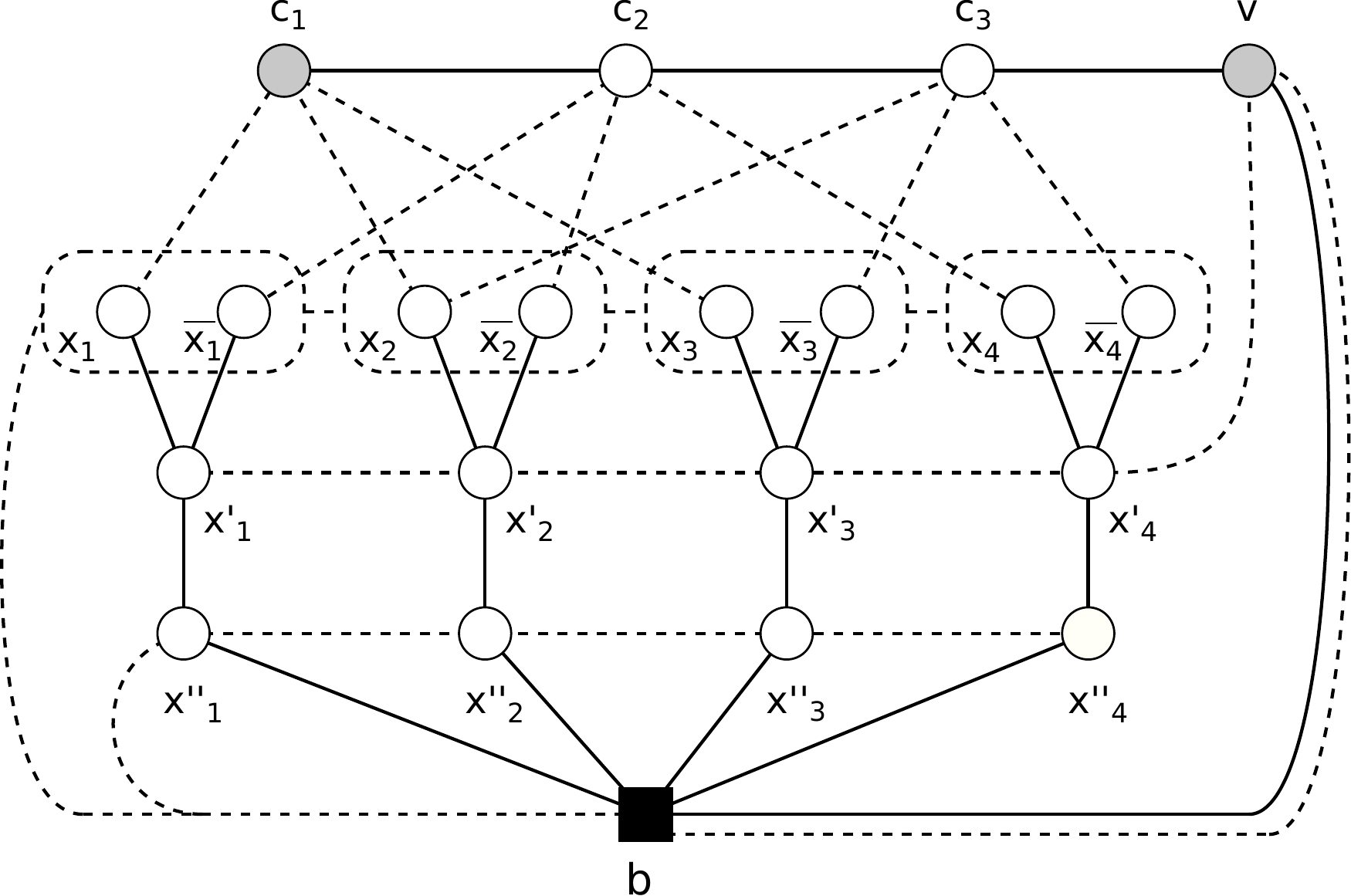}
	\caption{Example of a transformation from the 3SAT instance $\{c_1=\{x_1, x_2, x_3\}, c_2=\{\overline{x}_1, \overline{x}_2, x_4\}, c_3=\{x_2, \overline{x}_3, \overline{x}_4\}\}$ to a \mbox{d-CMPS} instance. The filled circles depict sensing locations, the solid lines depict edges from $E_M$ and the dashed lines depict edges from $E_C$. When there is a dashed edge between two dashed boxes, then there is an edge from $E_C$ between all pairs of vertices. There is no edge from $E_C$ within a dashed box.}
	\label{fig:cmps_3sat}
\end{figure}

\begin{proposition}
d-CMPS is NP-hard.
\end{proposition}

\begin{proof}
The transformation from an instance of 3SAT with variables $W=\{x_1, \ldots, x_\alpha\}$, and clauses $C=\{c_1, \ldots, c_\beta\}$ to an instance of \mbox{d-CMPS} with graphs $G_M$ and $G_C$ with $n=4\alpha+\beta+2$ vertices is defined as:
\begin{itemize}
	\item $r=\alpha+1 < n$
	\item $V=\{x_1, \ldots, x_\alpha, \overline{x}_1, \ldots, \overline{x}_\alpha, x'_1, \ldots, x'_\alpha, x''_1, \ldots, x''_\alpha, $\\$c_1, \ldots, c_\beta, b, v\}$
	\item $V_S=\{v, c_1\}$
	\item $[b,x''_i], [x''_i,x'_i], [x'_i, x_i], [x'_i, \overline{x}_i] \in E_M, 1 \leq i \leq \alpha$
	\item $[b, v], [v, c_\beta], [c_{i-1}, c_i] \in E_M, 2 \leq i \leq \beta$
	\item $[b,x''_1], [x''_{i-1},x''_i] \in E_C, 2 \leq i \leq \alpha$
	\item $[v,x'_\alpha], [x'_{i-1},x'_i] \in E_C, 2 \leq i \leq \alpha$
	\item $[b,z_1], [z_{i-1},z_i] \in E_C, 2 \leq i \leq \alpha$, $z_i \in \{x_i,\overline{x}_i\}$
	\item $[c_j,z_i] \in E_C$ if $z_i$ appears in $c_j$, $z_i \in \{x_i,\overline{x}_i\}$
	\item $p_0=(v, x'_1, \ldots, x'_\alpha)$
	\item $T=2\beta$
\end{itemize}

On one hand, a solution to the 3SAT instance gives a solution with $WI=2\beta-1$: A robot at vertex $x'_i$ moves to $\overline{x}_i$ or $x_i$ depending on the assignment of variable $x_i$. Since all clauses $c_j$ are satisfied, robot 1 in $v$ can move to $c_1$ and back to $v$. When the robot is in $v$, the other robots can move back to the $x'_i$ vertices.

On the other hand, every solution that results in a $WI=2\beta-1$ must result in an assignment for the variables $x_i$ that satisfies all clauses $c_j$. The only possibility to reach any $z_i$ is from the start position. If robot 1 moves to $c_\beta$, a robot in $x'_i$ has to move to a $z_i$ to keep robot 1 connected to the base station, and if robot 1 reaches $c_i$, the positions of the other robots constitutes a satisfying assignment for the 3SAT instance.
\end{proof}

The $\alpha$ vertices $x''_i$ above the base station are not necessary to show the NP-hardness of \mbox{d-CMPS}. Nevertheless, the use of these vertices can show that the problem cannot be approximated with any constant factor unless $P=NP$. The structure of the vertices $x'_i$, $x''_i$ and $v$ prevents that multiple robots gather at some $x'_i$ and move to some $x_i$ and $\overline{x}_i$ simultaneously or let the robots change the assignment while robot 1 is commuting between $c_1$ and $v$, which would solve the problem without solving the 3SAT instance. No matter how large the time bound $T$ is, commuting between $c_1$ and $v$ can only be done by solving the 3SAT instance.

\begin{figure}[t]
	\centering
	\includegraphics[scale=0.4]{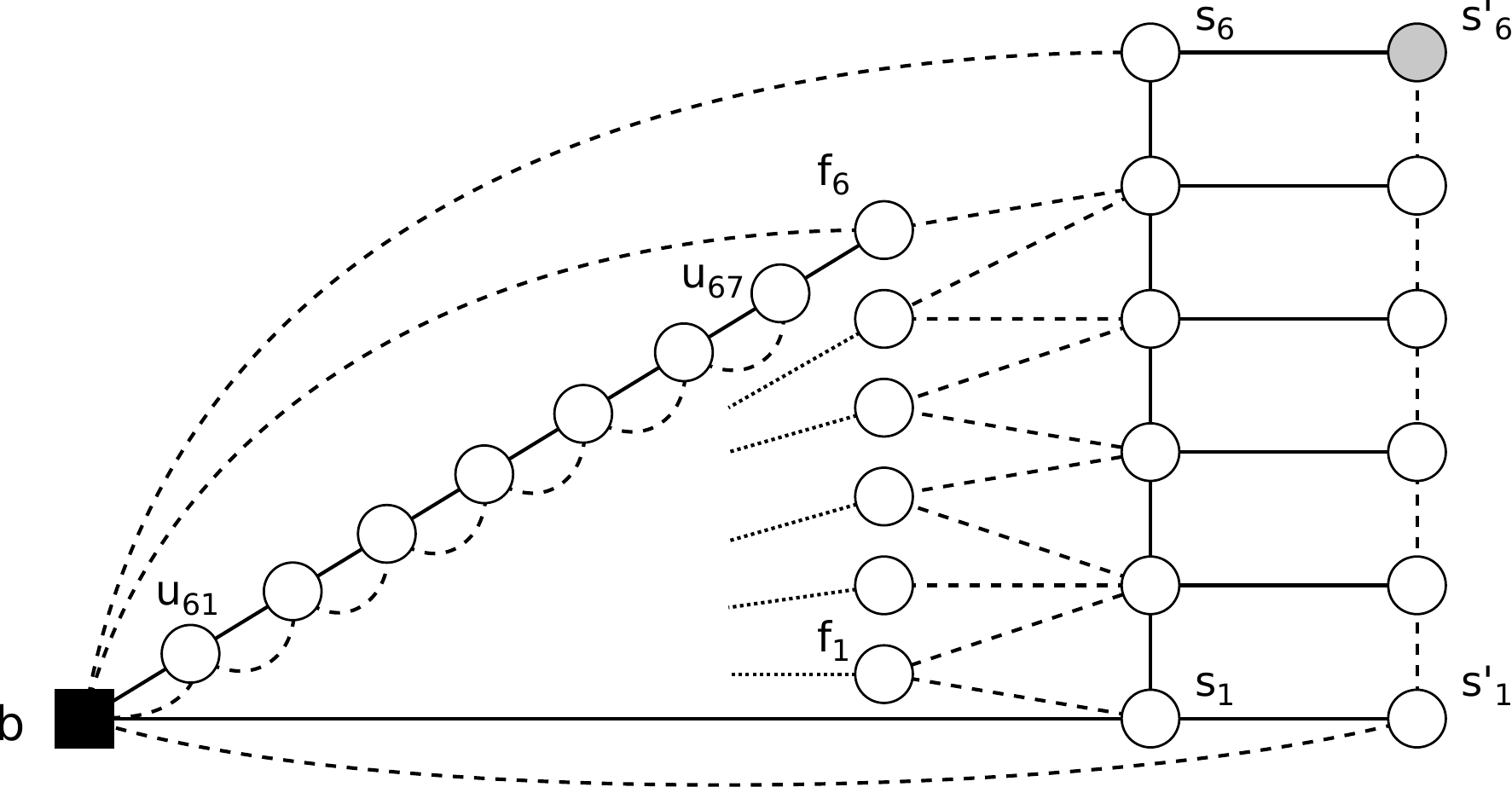}
	\caption{Example of a transformation from the SC instance $S=\{1, \ldots, 5\}, F=\{f_1=\{1,2\},f_2=\{2\},f_3=\{2,3\},f_4=\{3,4\},f_5=\{4,5\},f_6=\{5\}\}$ to a \mbox{d-CMR} instance. The filled circles depict sensing locations, the solid lines depict edges from $E_M$, the dashed lines depict edges from $E_C$. Only the path from $b$ to $f_6$ is shown, the paths between $b$ and $f_1$ through $f_5$ are indicted with a dotted line for better readability.}
	\label{fig:cmr_sc}
\end{figure}

In practical scenarios a predefined number of robots starts at a base station, which raises two related questions: How many robots are at least necessary to reach a certain vertex, and can each sensing location be reached from the base station? We define the decision problem \mbox{d-CMR} to show that the problem of determining the minimum number of robots necessary to reach a single vertex (as well as determining the minimum number of time steps to do so), when all robots start at the base station, is NP-hard (NP-complete). The second question is treated in the following Subsection~\ref{sec:traversal}.
\begin{definition}[d-CMR]
Problem \mbox{d-CMR} is a set of tuples of the form $(G_M,G_C,b,g,r,1^T)$ where $g$ is the goal vertex and $T$ is a time bound. The elements of the set have the following properties: (i)~There is a sequence $(p_0, \ldots, p_{t'}), t' \leq T$ such that $p_0 = (b, \ldots, b)$, $p_{t'}(i)=g$ for some robot $i$, (ii)~the vertex induced subgraph $G_C^t$ is connected for $0 \leq t \leq t'$, and (iii)~${r<n}$.
\end{definition}

To show that \mbox{d-CMR} is NP-hard we provide a transformation from set cover\footnote{An instance of SC consists of a set $S=\{s_1, \ldots, s_\alpha\}$, a subset family $F=\{f_1, \ldots, f_\beta\}$, $f_i\subseteq S$, and a number $k$. The question is, whether there are at most $k$ subsets from $F$ such that their union is equal to $S$.} (SC) to \mbox{d-CMR}. The transformation from an SC instance with elements $S=\{1, \ldots, 5\}$ and subset family $F=\{f_1=\{1,2\},f_2=\{2\},f_3=\{2,3\},f_4=\{3,4\},f_5=\{4,5\},f_6=\{5\}\}$ is shown in Figure~\ref{fig:cmr_sc}. The general transformation is described in the proof of

\begin{proposition}
\mbox{d-CMR} is NP-hard.
\end{proposition}

\begin{proof}
The transformation from an instance of SC with elements $S=\{s_1, \ldots, s_\alpha\}$, a subset family $F=\{f_1, \ldots, f_\beta\}$, and a number $k$ (the trivial case $k \geq \beta$ can be ignored) to an instance of \mbox{d-CMR} with $n=1+(M+2)\beta+2M$ is defined as (with $M:=\max{\{\alpha, \beta\}}$):
\begin{itemize}
	\item $r=M+k < M+\beta < n$
	\item $V=\{b, u_{11}, \ldots, u_{1(M+1)}, \ldots, u_{\beta 1}, \ldots, u_{\beta(M+1)},$ \\ $f_1, \ldots, f_\beta, s_1, \ldots, s_M, s'_1, \ldots s'_M\}$
	\item $[b,u_{i1}], [u_{i(M+1)},f_i]\in E_M, 1\leq i \leq \beta$
	\item $[u_{ij},u_{i(j+1)}] \in E_M, 1\leq i \leq \beta, 1 \leq j \leq M$
	\item $[b,s_1], [s_j,s_{j+1}] \in E_M, 1 \leq j \leq M-1$
	\item $[s_j, s'_j] \in E_M, 1 \leq j \leq M$
	
	\item $[b,u_{i1}], [b,f_i]\in E_C, 1\leq i \leq \beta$
	\item $[u_{ij},u_{i(j+1)}] \in E_C, 1\leq i \leq \beta, 1 \leq j \leq M$
	\item $[b,s'_1], [s'_j,s'_{j+1}] \in E_C, 1 \leq j \leq M-1$
	\item $[b, s_M] \in E_C$
	\item $[f_j,s_i] \in E_C$ if $s_i \in f_j$
	\item $g=s'_M$
	\item $T=2\beta(M+1)+M+1$
\end{itemize}

A solution of the SC instance with $S=f_{j_1} \cup \ldots \cup f_{j_l}$ and $l \leq k$ results in the following solution for the instance of \mbox{d-CMR}: $l$ robots are placed at vertices $f_{j_1}, \ldots, f_{j_l}$ subsequently with help of $M$ robots that build a chain along the vertices $u_{j1}$ to $u_{jM}$ on a path to a vertex $f_j$ ($j \in \{j_1, \ldots, j_l\}$) and return to $b$. The sequence for placing one robot at $f_6$ in the example in Figure~\ref{fig:cmr_sc} is given as follows. Note that only the positions for the $M+1=7$ robots, which are required to reach $f_6$, are shown:

\begingroup
\setlength{\tabcolsep}{3pt} 
\renewcommand{\arraystretch}{1} 
\begin{tabular}{lllllllll}
	(\ldots, & $b$,      &$b$,      &$b$,      &$b$,      &$b$,      &$b$,      &$b$,      &\ldots) \\
	(\ldots, & $u_{61}$, &$u_{61}$, &$u_{61}$, &$u_{61}$, &$u_{61}$, &$u_{61}$, &$u_{61}$, &\ldots) \\
	(\ldots, & $u_{61}$, &$u_{62}$, &$u_{62}$, &$u_{62}$, &$u_{62}$, &$u_{62}$, &$u_{62}$, &\ldots) \\
	(\ldots, & $u_{61}$, &$u_{62}$, &$u_{63}$, &$u_{63}$, &$u_{63}$, &$u_{63}$, &$u_{63}$, &\ldots) \\
	& & & & \ldots & & & \\
	(\ldots, & $u_{61}$, &$u_{62}$, &$u_{63}$, &$u_{64}$, &$u_{65}$, &$u_{66}$, &$u_{67}$, &\ldots) \\
	(\ldots, & $u_{61}$, &$u_{62}$, &$u_{63}$, &$u_{64}$, &$u_{65}$, &$u_{65}$, &$f_{6}$,  &\ldots) \\
	(\ldots, & $u_{61}$, &$u_{62}$, &$u_{63}$, &$u_{64}$, &$u_{64}$, &$u_{64}$, &$f_{6}$,  &\ldots) \\
	& & & & \ldots & & & \\
	(\ldots, & $b$,      &$b$,      &$b$,      &$b$,      &$b$,      &$b$,      &$f_{6}$,  &\ldots)
\end{tabular}
\endgroup

After $l$ robots have been placed at $f_{j_1}, \ldots, f_{j_l}$, each vertex $s_i$ is connected to the base station (since the set $S$ is covered) and can be occupied by a robot. In the last step all these $M$ robots move to the vertices $s'_i$ at the same time along the edges $[s_i, s'_i] \in E_M, 1 \leq i \leq M$.

There are $M$ robots necessary to place a robot at a vertex $f_j$ or to change the position of a robot from a vertex $f_{j_1}$ to another vertex $f_{j_2}$. Therefore, at most $k-1$ robots can be placed at vertices $s'_1$ to $s'_{k-1}$, and since $k<M$, the vertex $s'_M$ cannot be reached by changing the positions of robots at vertices $f_j$. This means that before the last robot can move to the goal $s'_M=g$, all vertices $s_1$ to $s_\alpha$ must be connected to the base station. This is only possible if at most $k$ robots have been placed at vertices $f_{j_1}, \ldots, f_{j_l}$, and a solution for the SC instance is determined by the occupied vertices $f_{j_1}, \ldots, f_{j_l}$ at that moment when another robot reaches the goal vertex. 
\end{proof}

A natural strategy for the reachability problem when ${E_M \subseteq E_C}$ is that a group of robots start at the base station and move on a path from the base station to the goal in $G_M$ until the connectivity to the base station would break when moving further. Every time when this happens, a relay robot stays at the current vertex and the remaining robots continue moving towards the goal \cite{Mosteo2009}. Determining the optimal relay positions such that a goal can be reached with a predefined number of robots on a predefined path to the goal is also NP-hard, which we show with a transformation from 3SAT to \mbox{d-CMRD} (see Proposition~\ref{prop:cmrd} and Figure~\ref{fig:cmrd_3sat}) defined in

\begin{definition}[d-CMRD]
The problem \mbox{d-CMRD} is a set of tuples of the form $(P, E_C, r)$ where $P:=(b, v_1, \ldots, v_n)$ is a sequence of vertices that describe a movement path, $E_C$ denotes the connectivity edges between the vertices (at least $[b,v_1], [v_{i-1},v_i] \in E_C$), and $r$ is the number of robots. The elements of the set have the property that $v_n$ can be reached with placing $r-1$ robots at relay positions at some $v_i$.
\end{definition}

\begin{proposition}
\label{prop:cmrd}
\mbox{d-CMRD} is NP-hard.
\end{proposition}

\begin{proof}
The transformation from an instance of 3SAT with variables $W=\{x_1, \ldots, x_\alpha\}$, and clauses $C=\{c_1, \ldots, c_\beta\}$ to an instance of \mbox{d-CMPS} with graphs $G_M$ and $G_C$ with $n=3\alpha+\beta+1$ vertices is defined as:
\begin{itemize}
	\item $r=\alpha+1$
	\item $P=(b, x_1, \overline{x}_1, \ldots, x_\alpha, \overline{x}_\alpha, x'_1, \ldots, x'_\alpha, c_1, \ldots, c_\beta)$
	\item $[\overline{x}_\alpha,x'_1], [x'_{i-1},x'_i] \in E_C, 2 \leq i \leq \alpha$
	\item $[x'_\alpha,c_1], [c_{i-1},c_i] \in E_C, 2 \leq i \leq \beta$
	\item $[b,x_1], [\overline{x}_{i-1},x_i] \in E_C, 2 \leq i \leq \alpha$
	\item $[x_i,\overline{x}_i] \in E_C, 1 \leq i \leq \alpha$
	\item $[z_{i-1},z_i] \in E_C, 2 \leq i \leq \alpha$, $z_i \in \{x_i,\overline{x}_i\}$
	\item $[z_i,x'_i] \in E_C, 1 \leq i \leq \alpha$, $z_i \in \{x_i,\overline{x}_i\}$
	\item $[c_j,z_i] \in E_C$ if $z_i$ appears in $c_j$, $z_i \in \{x_i,\overline{x}_i\}$
\end{itemize}
A solution of the 3SAT instance defines the positions of the $\alpha$ relays such that the robot $\alpha+1$ can reach the goal vertex~$c_\beta$.

If a robot can reach $c_\beta$, then the vertices $x'_1, \ldots, x'_\alpha$, $c_1, \ldots, c_\beta$ are connected to the base station, which is only possible if a relay is placed at each $z_i$ such that the 3SAT instance is satisfied.
\end{proof}

\begin{figure}[t]
	\centering
	\includegraphics[scale=0.4]{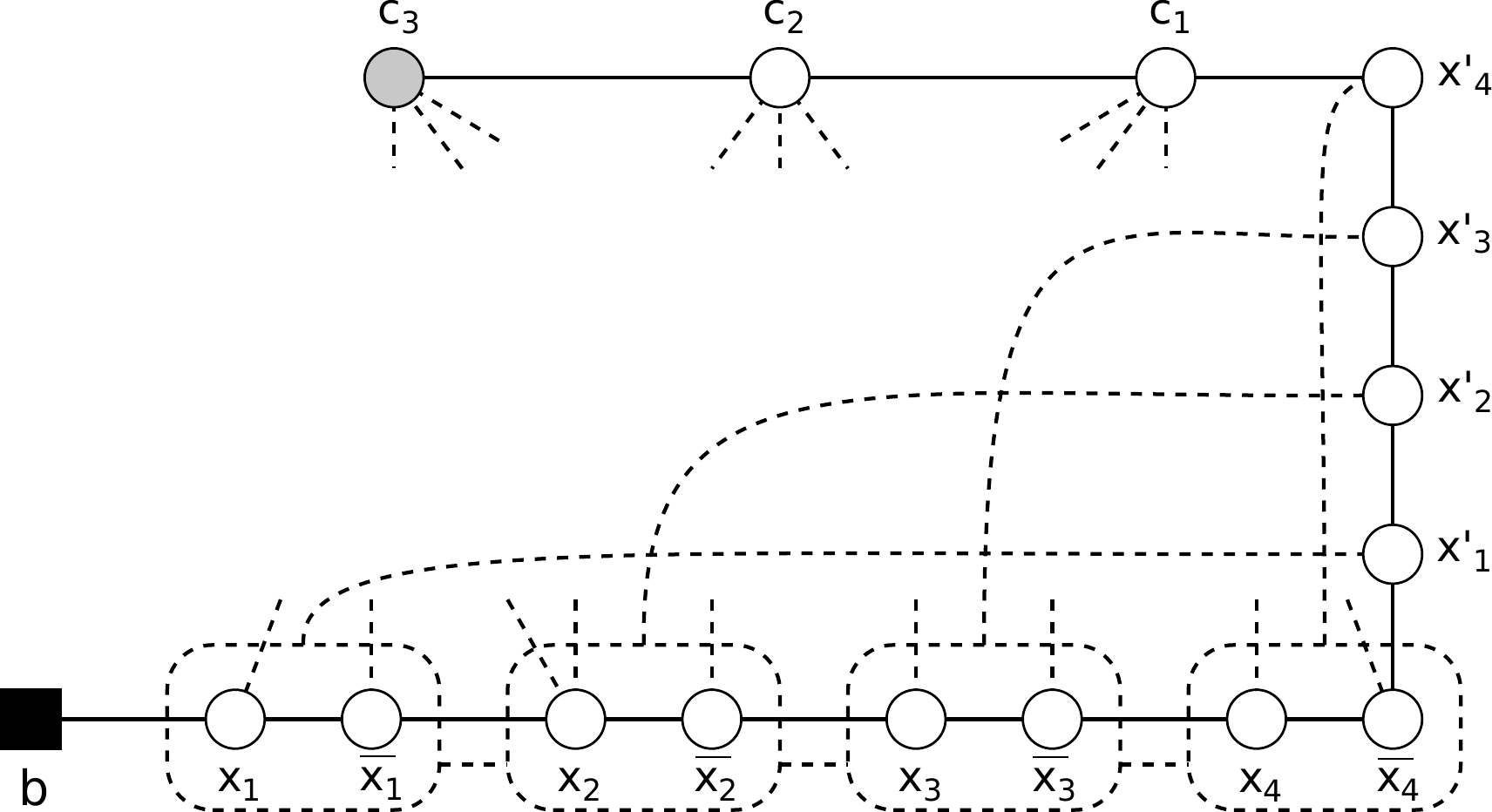}
	\caption{Example of a transformation from a 3SAT instance (same example as in Figure~\ref{fig:cmps_3sat}) to a \mbox{d-CMRD} instance. When there is a dashed edge between two dashed boxes, then there is an edge from $E_C$ between all pairs of vertices. There is also an edge from $E_C$ between the vertices in a dashed box (since ${E_M\subseteq E_C}$).}
	\label{fig:cmrd_3sat}
\end{figure}

\subsection{Note on graph traversal}
\label{sec:traversal}

A precondition for the existence of a solution to a persistent surveillance problem (where the robots start at the base station) is that all sensing locations can be reached from the base station with the available number of robots. In the previous subsection we have shown that determining the minimum number of time steps to reach a goal vertex from the base station is NP-hard. In this subsection we argue that determining the minimum number of robots necessary to reach a goal vertex is also NP-hard. We do not attempt to determine a solution with the minimum number of robots but provide a necessary and sufficient condition for the existence of a solution with $n-1$ robots, since there are graphs for which no solution exists even if $G_M$ and $G_C$ is connected.

We say that a graph can be traversed with $r$ robots if there is a solution such that every sensing location can be visited with $r$ robots starting at the base station. Although a graph cannot be traversed if sensing locations and the base station are at different connected components of $G_M$ or $G_C$, the fact that $G_M$ and $G_C$ are connected, is not a sufficient condition for a graph to be traversable. Such an example is shown in Figure~\ref{fig:example_simple_1} where both $G_M$ and $G_C$ are connected but the graph cannot be traversed. If a robot moves from the base station to vertex $v$ or $w$ (because $[b,v] \in E_M$ and $[b,w] \in E_M$), it will be disconnected from the base station. The graph in Figure~\ref{fig:example_simple_2} can be traversed with two robots but not with one (although the distance in $G_C$ is $dist_{G_C}(b,w)=1$). An example sequence for traversal is $p_0=(b,b), p_1=(u,u), p_2=(u,v), p_3=(u,w)$.

\begin{figure}[t]
	\centering
	\begin{tabular}{cc}
		\subfloat[]{
			\includegraphics[scale=0.4]{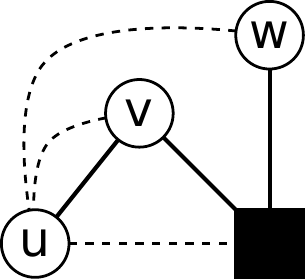}
			\label{fig:example_simple_1}
		}
		&
		\subfloat[]{
			\includegraphics[scale=0.4]{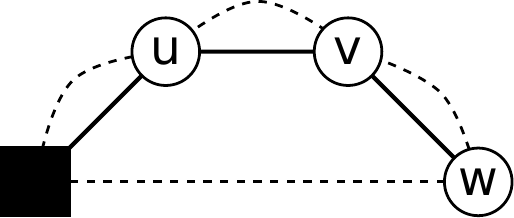}
			\label{fig:example_simple_2}
		}
	\end{tabular}
	\caption{Simple examples for graph traversals. The graph in (a) cannot be traversed although $G_M$ and $G_C$ are connected. The graph in (b) can be traversed with two but not with one robot.}
\end{figure}

\begin{figure}[t]
	\centering
	\begin{tabular}{cc}
		\subfloat[]{
			\includegraphics[scale=0.4]{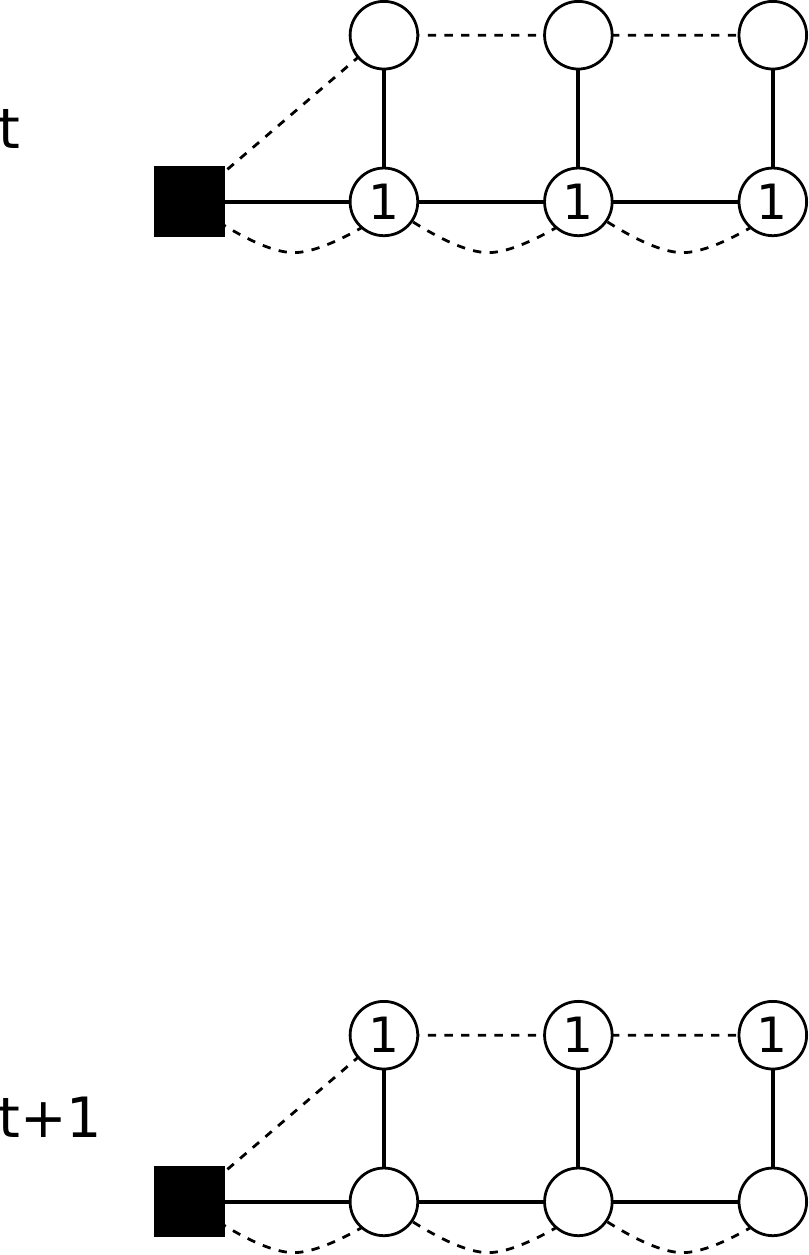}
			\label{fig:example_traverse_1}
		}
		&
		\subfloat[]{
			\includegraphics[scale=0.4]{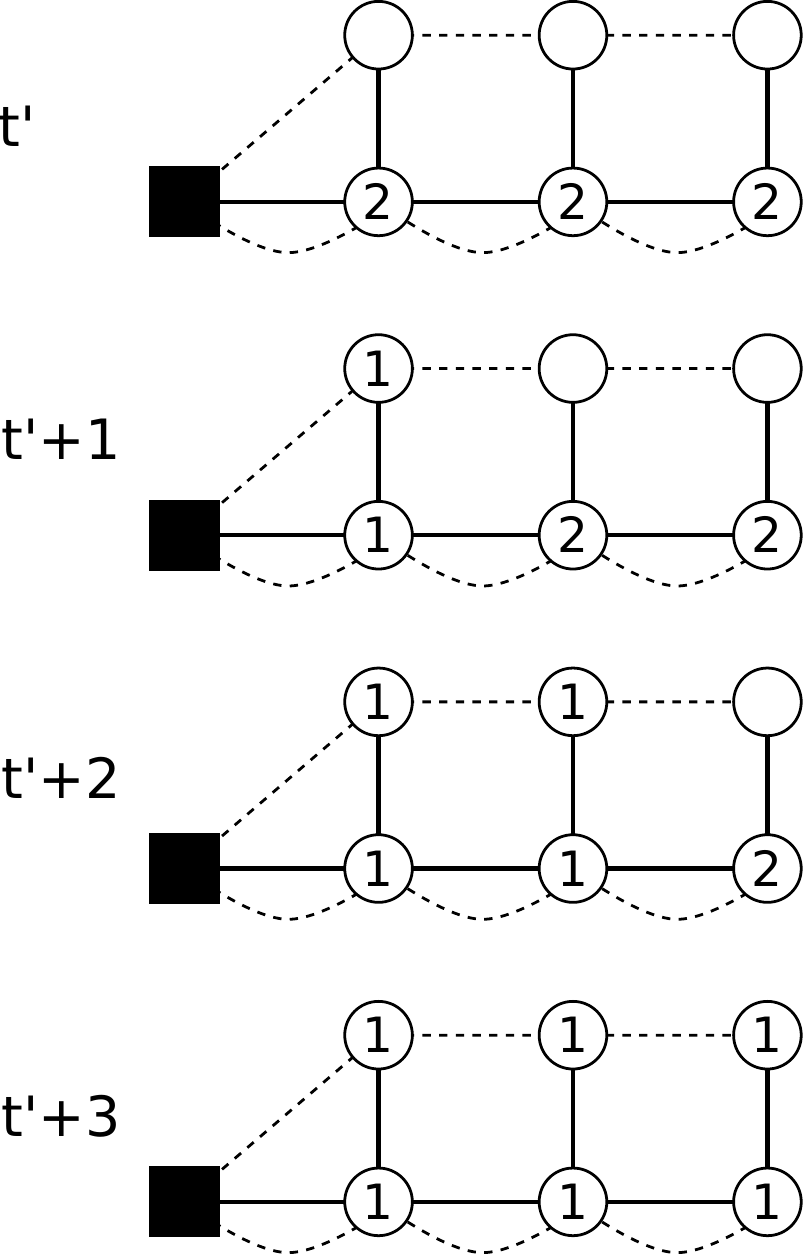}
			\label{fig:example_traverse_2}
		}
	\end{tabular}
	\caption{Illustration for the proof of Proposition~\ref{prop:traverse}. The graph in (a) is traversed by three robots in two time steps (from top to bottom), which move from the bottom to the top vertices simultaneously. This sequence corresponds to solution $P$ in the text. In (b) the order in which the vertices are marked by Algorithm~\ref{alg:traverse} are shown in four steps from top to bottom, which can be interpreted as traversal with 6 robots. This sequence corresponds to solution $P'$ in the text. The numbers in the vertices indicate the number of robots that are at a vertex at a particular time step.}
	\label{fig:example_traverse}
\end{figure}

Algorithm~\ref{alg:traverse} determines whether a graph can be traversed with $n-1$ robots. We show that it terminates with $traverse=true$ if and only if the graph can be traversed with $n-1$ robots. Although using $n-1$ robots is a trivial solution, determining a solution with the minimum number of robots such that a graph can be traversed also solves the SC instance of \mbox{d-CMR}. This is because a solution with a minimum number of $r^*$ robots corresponds to a set cover of cardinality $r^*-M$: If the goal $s'_M$ can be reached with $r^*$ robots, then $M$ robots are necessary for the chain $s'_1, \ldots, s'_M$, and $r^*-M$ robots are at vertices $f_i$ such that all $s_1, \ldots, s_M$ are connected to $b$.

\begin{algorithm}[t]
	\caption{Traverse}
	\label{alg:traverse}
	\small
	\begin{algorithmic}[1] 
		\Require
			\Statex \begin{flushleft} $V_S$, $G_M$, $G_C$, $b$\end{flushleft}
		\Ensure
			\Statex \begin{flushleft} $parent(\cdot), traverse$ \end{flushleft}
		\State $parent(v) \gets null, marked(v) \gets false, \forall v \in V$
		\State $marked(b) \gets true$
		\State $added \gets true$
		\While{$added$}
			\State $added \gets false$
			\For{$v \in V$}
				\If{$!marked(v)$}
					\State $p \gets null, c \gets false$
					\If{$\exists [v,w] \in E_M$ s.t. $marked(w)$}
						\State $p \gets w$
					\EndIf
					\If{$\exists [v,u] \in E_C$ s.t. $marked(u)$}
						\State $c \gets true$
					\EndIf
					\If{$p \neq null$ and $c$}
						\State $marked(v) \gets true$
						\State $parent(v) \gets p$
						\State $added \gets true$
					\EndIf
				\EndIf
			\EndFor
		\EndWhile
		\State $traverse \gets |\{v \in V_S: marked(v)\}| = |V_S|$
	\end{algorithmic}
\end{algorithm}

\begin{proposition}
\label{prop:traverse}
Algorithm~\ref{alg:traverse} terminates with $traverse=true$ if and only if the graph can be traversed with $n-1$ robots.
\end{proposition}

\begin{proof}
First we show that if $v$ is marked, then $v$ can be visited. This can be shown by induction with the assumption that every vertex $v$ that is added is occupied by a robot $i$ that started at the base station (i.e. $p_0(i)=b$) and has connectivity to the base station (there is a path in $G_C^t$ from $b$ to $v$). For the first vertex $v$ that is added there must be edges $[b,v] \in E_M$ and $[b,v] \in E_C$, i.e. robots can move from $b$ to $v$. For subsequently marked vertices $v$ there is an edge $[v,w] \in E_M$ such that a subset of the robots that are at $w$ can move to $v$. The connectivity is maintained by a robot at vertex $u$ and the edge $[v,u] \in E_C$. The traversal can be interpreted as a solution with $n-1$ robots: If a vertex $v$ is marked, a number of robots (sufficient to visit all vertices that are marked in subsequent steps) move from $w$ to $v$ such that there is exactly one robot at each $v$ that has been marked when the algorithm terminates. The number of robots that should move to a marked vertex can be determined by traversing the tree described by the $parent$ function, which is not shown here.

To show that if a solution $P$ with $r$ robots traverses a graph, the algorithm terminates with $traverse=true$, a solution $P'$ with $n-1$ robots can be constructed with help of the following observation. From two subsequent positions $p_t$ and $p_{t+1}$ from $P$, where $k$ robots change their position, a sequence of at most $k+1$ positions $p'_{t'}, \ldots, p'_{t'+l}$ with $l \leq k$ and $\{p_t(1) \cup \ldots \cup p_t(r)\} \subseteq \{p'_{t'}(1) \cup \ldots \cup p'_{t'}(n)\}$ such that $\{p_{t+1}(1) \cup \ldots \cup p_{t+1}(r)\} \subseteq \{p'_{t'+l}(1) \cup \ldots \cup p'_{t'+l}(n)\}$ can be constructed. Figure~\ref{fig:example_traverse_1} shows two steps of a solution $P$ with $r=3$ where the robots move at the same time from the lower to the upper vertices. In Figure~\ref{fig:example_traverse_2} four steps of a solution $P'$ with $r=n-1=6$ are shown, which visits the same set of vertices as $P$. First, for a robot $i_1$ in $P$ for which $[p_{t+1}(i_1), b] \in E_C$, a robot $j_1$ in $P'$ can move from $p'_{t'}(j_1)=p_{t}(i_1)$ to $p'_{t'+1}(j_1)=p_{t+1}(i_1)$. In particular, $p_{t+1}(i_1)$ is marked, and these are the robots that move to the upper left vertex in Figures~\ref{fig:example_traverse} (a) and (b). Then, for a robot $i_2$ with $[p_{t+1}(i_1), p_{t+1}(i_2)] \in E_C$ a robot $j_2$ can move in $P'$ such that $p'_{t'+2}(j_2)=p_{t+1}(i_2)$. These are the robots that move to the upper middle vertex in Figure~\ref{fig:example_traverse}. This procedure can be continued to $p'_{t'+l}(j_l)=p_{t+1}(i_l)$.

We have to show that Algorithm~\ref{alg:traverse} marks all vertices that can be visited by the solution $P$. Suppose that $P$ traverses the graph but Algorithm~\ref{alg:traverse} terminates with $traverse=false$. Then there must be two steps of $P$, $p_t$ and $p_{t+1}$, such that $\{p_t(1) \cup \ldots \cup p_t(r)\} \subseteq \{v \in V: marked(v)\}$ (this is certainly true for $t=0$) and $\{p_{t+1}(1) \cup \ldots \cup p_{t+1}(r)\} \nsubseteq \{v \in V: marked(v)\}$. However, this contradicts the observation above.
\end{proof}

\section{Persistent surveillance on convex grid areas}
\label{sec:convex}

We consider a convex mission area without obstacles, which is divided into a two-dimensional grid of square cells with unit side length. A subset of these cells are sensing locations and the base station is located at a particular cell. A robot can move from a cell to one of the eight neighboring cells (except for cells at the boundary of the area) within one time step. The cells correspond to vertices of $G_M$ and $G_C$, and $E_M$ contains the edges between neighboring cells. The edges in $E_C$ are defined by the communication range $R^{com}$, which is measured in cells. For example, $R^{com}=3$ means that two cells are in communication range if the Euclidean distance between the center of the cells is smaller than or equal to $3$. For the sake of completeness we recap the algorithms for short horizon (SH) and short horizon cooperative (SHC) movement planning \cite{Scherer2017}.

\subsection{Short horizon (SH) movement planning}

The policy in \cite{Nigam2012} assigns a sensing location for each robot $i$ at each time step $t$ based on a weighted combination of the instantaneous idleness $I_t(v)$ of sensing location $v$, the distance between the robot and the sensing location $dist_{G_M}(p_t(i), v)$, and the minimum distance between $v$ and any other robot $j \neq i$:
\begin{align}
\label{eq:assign}
	A(v, i) &= I_t(v) \nonumber \\
	&+ \omega_0 \: dist_{G_M}(p_t(i), v) \nonumber \\
	&+ \omega_1 \min_{j \neq i}{\{dist_{G_M}(p_{t}(j), v)\}}.
\end{align}
Each robot gets assigned to the sensing location with the highest value $A(i, v)$ individually. The weighting parameters $\omega_0$ and $\omega_1$ are determined by an offline optimization algorithm where the parameter space is sampled and the mission is simulated to get the objective value for a particular set of parameters. We adopt this approach to enforce the connectivity constraints by disallowing moves that would result in a disconnected network and denote it as \emph{short horizon} (SH) movement planner (Algorithm~\ref{alg:sh}). The algorithm calculates positions for the robots starting at the base station over a finite horizon of $T$ time steps. After calculating the assignment matrix according to Equation~\ref{eq:assign} (line~\ref{line:sh_amatrix}), the goal for each robot is determined (line~\ref{line:sh_nextgoal}). The position of a robot $i$ at time $t+1$ is the neighboring position of the position at time $t$ that is closest to the goal and does not disconnect the network of the robots (line~\ref{line:sh_nextpos}). $N(p_t(i))$ denotes the set of neighbor positions of the current position $p_t(i)$ of robot $i$ including the current position. The goals and new positions for time $t$ are calculated for each robot consecutively in arbitrary order.

\begin{algorithm}[t]
	\caption{Short horizon (SH)}
	\label{alg:sh}
	\small
	\begin{algorithmic}[1] 
		\Require
			\Statex \begin{flushleft} $T$, $R$, $V_S$, $G_M$, $G_C$, $b$, $\omega_0, \omega_1$ \end{flushleft}
		\Ensure
			\Statex \begin{flushleft} $P=(p_0, \ldots, p_T)$ \end{flushleft}
		\State $p_0 \gets (b, \ldots, b)$
		\For{$t \gets 0, \ldots, T-1$}
			\State $p' \gets p_{t}$
			\State $A \gets \Call{CalcAssignmentMatrix}{}$	\label{line:sh_amatrix}
			\For{$i \in R$}
				\State $goal(i) \gets \argmax_{v \in V_S}{\{A(v, i)\}}$	\label{line:sh_nextgoal}
				\State $p'(i) \gets \argmin_{v \in N(p_t(i))}{\{dist_{G_M}(v, goal(i)):}$		\label{line:sh_nextpos}
				\StatexIndent[2] $G_C\langle\{p'(j): j \in R\setminus \{i\}\} \cup \{v\} \cup \{b\}\rangle \text{ is connected} \}$
			\EndFor
			\State $p_{t+1} \gets p'$
		\EndFor
		
		\Statex
		\Procedure{CalcAssignmentMatrix}{}
			\For{$v \in V_S, i \in R$}
				\State $A(v, i) \gets I_{t}(v)$
				\StatexIndent[3] $ + \omega_0 \: dist_{G_M}(p_{t}(i), v)$
				\StatexIndent[3] $ + \omega_1 \min_{j \neq i}{\{dist_{G_M}(p_{t}(j), v)\}}$
			\EndFor
			\State \Return $A$
		\EndProcedure
	\end{algorithmic}
\end{algorithm}

\subsection{Short horizon cooperative (SHC) movement planning}

Because the goals are approached independently by the robots in the SH algorithm, it can happen that robots block each other infinitely due to the connectivity constraint. To overcome this mutual blocking problem we developed an extension to SH based on graph matching and formation reconfiguration. SH assigns a goal to every robot, which then approach their goal individually. In contrast to that an iteration of SHC (Algorithm~\ref{alg:shc}) consists of three phases: (i) goal selection, (ii) goal assignment and (iii) reconfiguration.

\subsubsection{Goal selection}
The goal selection phase starts with calculating an assignment matrix according to Equation~\ref{eq:assign} (line~\ref{line:shc_amatrix}). After that the sensing locations with the highest values together with the base station are selected as terminal vertices for a Steiner tree in $G_C$. A Steiner tree is a vertex induced subgraph $G_C\langle \{b, v_1, \ldots v_i\} \cup N\rangle$ containing the terminal vertices $\{b, v_1, \ldots v_i\}$ and possibly a set of non-terminal vertices $N$ such that the tree is connected. The function $calc\_steiner\_tree$ calculates an approximate solution $T_f$ for the NP-hard Steiner tree problem with minimum number of non-terminal vertices \cite{Klein1995}. The algorithms tries to include as many sensing locations into a Steiner tree that can be built with the available number of robots (lines~\ref{line:shc_goalbegin} to \ref{line:shc_goalend}).

\subsubsection{Goal assignment}
The goal assignment phase performs a matching between the desired final tree $T_f$ and the actual vertex induced subgraph $G_a=G_C\langle\{p_{t}(i): i \in R\} \cup \{b\}\rangle$ containing the positions of the robots at time $t$ and the base station with the aim to reduce the time to reach the final configuration. A prefix labeling is applied to the vertices of the final tree $T_f$ \cite{Navaravong2012}, which results in the labels $L_f$ (line~\ref{line:shc_labelfinal}). Then a tree $T_a$ is selected from $G_a$ \cite{Kan2015}, and a matching between $T_a$ and $T_f$ is calculated \cite{Umeyama1988} (line~\ref{line:shc_graphmatching}). The labels for the vertices that could not be matched are determined (line~\ref{line:shc_completelabels}). The extra and missing vertices \cite{Navaravong2012} are computed in line~\ref{line:shc_calcextramissing}. Extra vertices represent robots and are vertices with labels from $L_a$ that are not in $L_f$ ($\mathcal{V}_e$). Missing vertices represent vertices from $V$ (final robot positions) with labels in $L_f$ that are not in $L_a$ ($\mathcal{V}_m$). Finally, a matching between extra and missing vertices is computed based on the distance in $T_a$ (line~\ref{line:shc_extramissing} to line~\ref{line:shc_emmatching}). Figure~\ref{fig:example_graph_matching} depicts an example for the prefix labeling and graph matching algorithm.

\begin{figure}[t]
	\centering
	\begin{tabular}{cc}
		\subfloat[]{
			\includegraphics[scale=0.45]{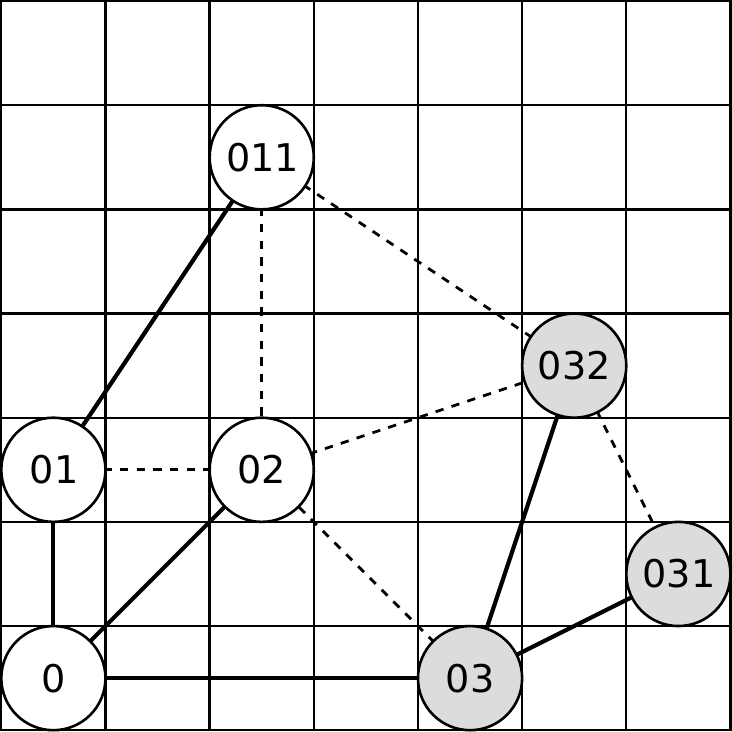}
			\label{fig:example_graph_matching_1}
		}
		&
		\subfloat[]{
			\includegraphics[scale=0.45]{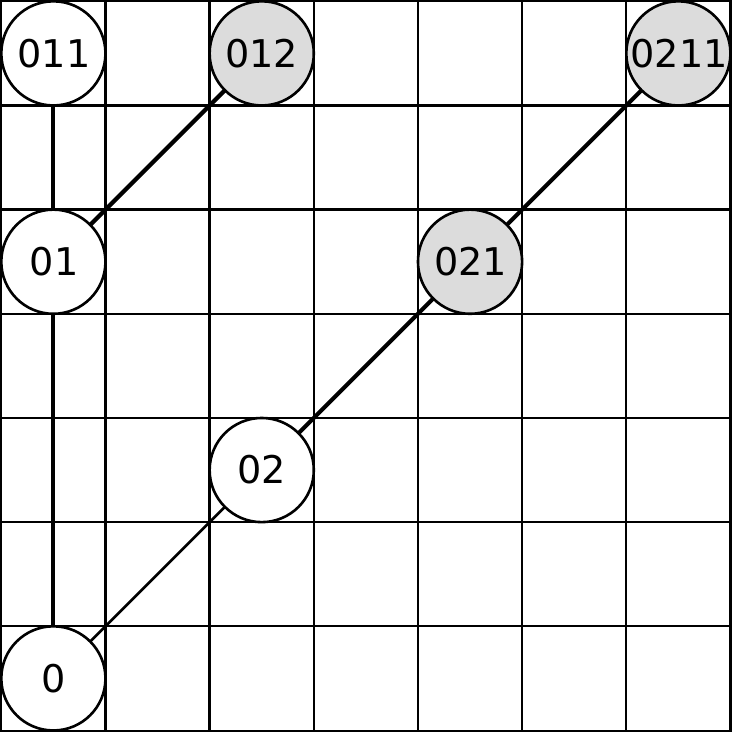}
			\label{fig:example_graph_matching_2}
		}
	\end{tabular}
	\caption{Outcome of the prefix labeling and graph matching algorithm for a grid area and 6 robots. The actual graph (a) models the state of the robots and the base station (node `0') at a particular time step with a line between two vertices if they are within communication range. A dashed line indicates the communication link and a solid line indicates the selected tree $T_a$ in the actual graph $G_a$. The graph (b) models the final desired robot configuration $T_f$, where each robot is assigned a relay position or a sensing location. Assume for example, that robots labeled with `011', `012', and `0211' are assigned to sensing locations (terminal vertices in the Steiner tree) but can reach them only with help of relays labeled with `01', `02', and `021' (nonterminal vetices). The shaded vertices are the extra vertices in (a) and the missing vertices in (b), respectively.}
\label{fig:example_graph_matching}
\end{figure}

\subsubsection{Reconfiguration}

In this phase every extra robot $i \in \mathcal{V}_e$ makes a move in $G_M$ towards the position $M(i)$ according the matching $M$ (line~\ref{line:shc_nextpos}). Similar to SH, line~\ref{line:sh_nextpos}, the next position is chosen such that the network does not get disconnected. When a robots reaches the goal position, the corresponding label in $L_a$ is updated to meet the label in $L_f$, and the robot is removed from the set of extra vertices (line~\ref{line:shc_reachedbegin} to \ref{line:shc_reachedend}). If the tree $T_f$ with the labels $L_f$ is a subgraph of the actual connectivity graph $G_a$ with the labels $L_a$, robot $i$ with label $L_a(i)$ gets assigned to the sensing location $v$ with the same label $L_f(v)$. Then the robot approaches its goal similar as in line~\ref{line:sh_nextpos} of SH, (line~\ref{line:shc_nextposbegin} to \ref{line:shc_nextposend}). Finally, if all sensing locations $v_1, \ldots, v_{\kappa}$ have been reached, the algorithm continues with the next goal selection phase (line~\ref{line:shc_nextphase}).

\begin{algorithm}[t]
	\caption{Short horizon cooperative (SHC)}
	\label{alg:shc}
	\small
	\begin{algorithmic}[1] 
		\Require
			\Statex \begin{flushleft} $T$, $R$, $V_S$, $G_M$, $G_C$, $b$, $\kappa$\end{flushleft}
		\Ensure
			\Statex \begin{flushleft} $P=(p_0, \ldots, p_T)$ \end{flushleft}
		\State $p_0 \gets (b, \ldots, b)$, $t \gets 0$
		\While{$t \leq T-1$}
			\State $A \gets \Call{CalcAssignmentMatrix}{}$	\label{line:shc_amatrix}
			\For{$i \in R$}
				\State $goal(i) \gets \argmax_{v \in V_S}{\{A(v, i)\}}$
				\State $value(i) \gets \max_{v \in V_S}{\{A(v, i)\}}$
			\EndFor
			\State $(v_1, \ldots, v_r) \gets$ sort $goal(\cdot)$ according to $value(\cdot)$ (descending)	\label{line:shc_goalbegin}
			\For{$i \gets 1, \ldots, r$}
				\State $T_f' \gets calc\_steiner\_tree(G_C, \{b, v_1, \ldots v_i\})$ (see \cite{Klein1995})
				\If{$|V(T_f')| > (r+1)$} break
				\EndIf
				\State $T_f \gets T_f'$	\label{line:shc_goalend}
			\EndFor
			\State $L_f \gets label\_trie(T_f)$ (see \cite{Navaravong2012})	\label{line:shc_labelfinal}
			\State $G_a \gets G_C\langle\{p_{t}(i): i \in R\} \cup \{b\}\rangle$
			\State $(T_a, L_a) \gets calc\_graph\_matching(G_a, T_f, L_f)$ (see \cite{Kan2015}, \cite{Umeyama1988}) \label{line:shc_graphmatching}
			\State $L_a \gets complete\_labels\_trie(T_a, L_a)$ (see \cite{Navaravong2012})	\label{line:shc_completelabels}
			\State $(\mathcal{V}_e, \mathcal{V}_m) \gets calc\_extra\_missing(L_a, L_f)$ (see \cite{Navaravong2012})\label{line:shc_calcextramissing}
			\For{$i \in \mathcal{V}_e$, $j \in \mathcal{V}_m$}	\label{line:shc_extramissing}
				\State $D(i, j) \gets dist_{T_a}(i, \text{closest parent of } j \text{ also in }T_a)$
			\EndFor
			\State $M \gets calc\_matching(D)$	\label{line:shc_emmatching}
			\For{$i \in \mathcal{V}_e$} $sp(i) \gets shortest\_path_{T_a}(i, M(i))$
			\EndFor
			\While{true}
				\State $p' \gets p_{t}$
				\For{$i \in R$}
					\If{$i \in \mathcal{V}_e$}
						\State $p'(i) \gets $ first step in $G_M$ along $sp(i)$	\label{line:shc_nextpos}
						\If{$p'(i) = M(i)$}	\label{line:shc_reachedbegin}
							\State $L_a \gets update\_labels(L_a)$
							\State $\mathcal{V}_e \gets \mathcal{V}_e \setminus \{i\}$		\label{line:shc_reachedend}
						\EndIf
					\EndIf
					\State $G_a \gets G_C\langle\{p'(i): i \in R\} \cup \{b\}\rangle$	\label{line:shc_nextposbegin}
					\If{$compare\_graphs(G_a, L_a, T_f, L_f)$}
						\State $w \gets v \in V_S$ with $L_f(v) = L_a(i)$
						\State $p'(i) \gets $ first step on 
						\StatexIndent[5] $shortest\_path_{G_M}(p'(i), w)$	\label{line:shc_nextposend}
					\EndIf
				\EndFor
				\State $p_{t+1} \gets p'$, $t \gets t+1$
				\If{all $v_1$ to $v_{\kappa}$ have been reached} break	\label{line:shc_nextphase}
				\EndIf
			\EndWhile
		\EndWhile
	\end{algorithmic}
\end{algorithm}

\subsection{Full horizon (FH) movement planning}

\begin{figure}[t]
	\centering
	\begin{tabular}{cc}
		\subfloat[]{
			\includegraphics[scale=0.45]{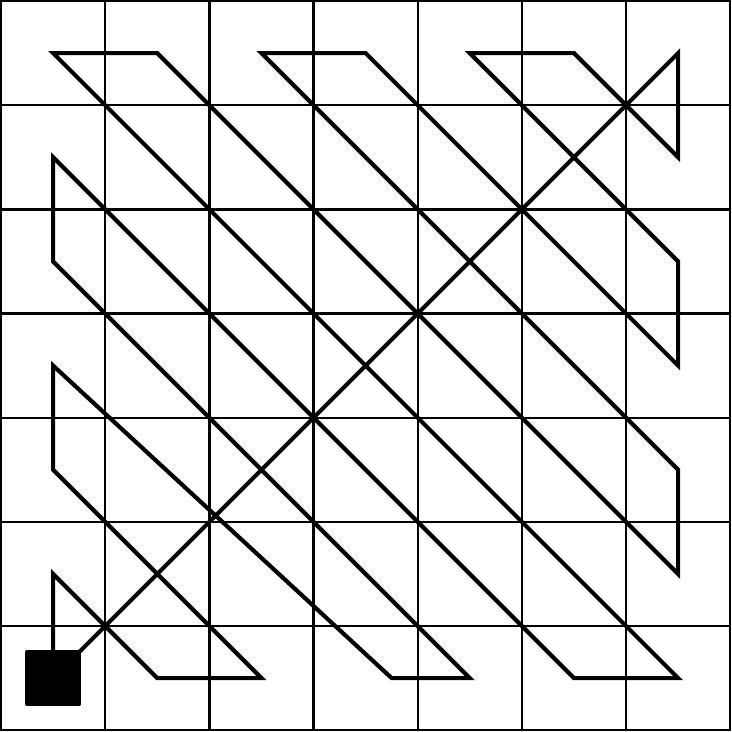}
			\label{fig:convexpath_bh}
		}
		&
		\subfloat[]{
			\includegraphics[scale=0.45]{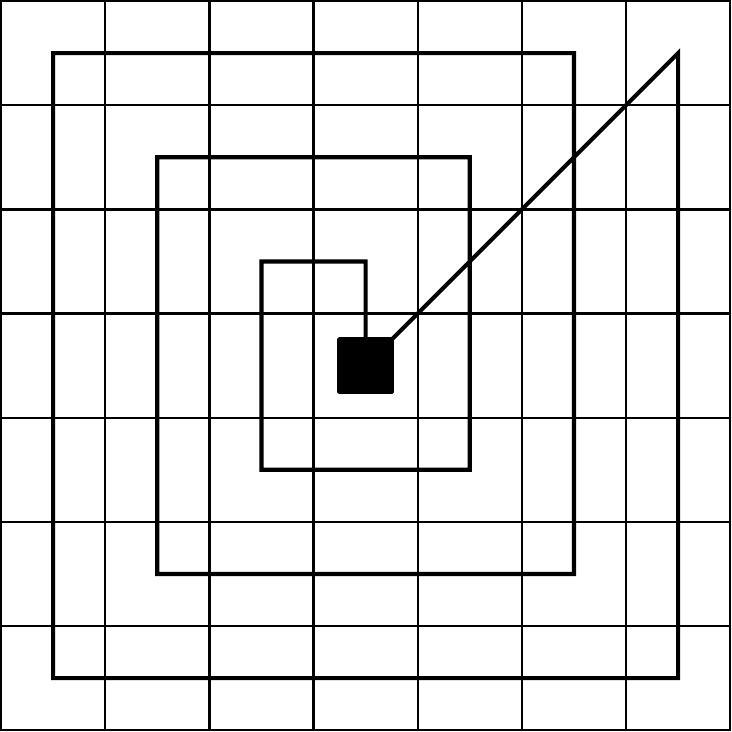}
			\label{fig:convexpath_sp}
		}
	\end{tabular}
	\caption{Two possibles coverage tours for a convex area when the base station is at a corner (a), or when the base station is near the center of the area (b). The tours start at the base station, pass through the upper right sensing location before continuing to cover the area towards the base station.}
\label{fig:convexpath}
\end{figure}

The full horizon approach requires a tour through all sensing locations. A leader robot traverses this path while the other robots maintain a chain in $G_C$ to the base station. When a relay robot is at a sensing location, it is considered visited and can be skipped from the leading robot's tour by taking a shortcut to the next unvisited sensing location on the tour. If a considerable proportion of the area consists of sensing locations, it is beneficial that the tour visits sensing location first that are farther away from the base station and that the relaying robots maintain an equal distance between each other on the chain. Two possible tours are depicted in Figure~\ref{fig:convexpath}. For persistent surveillance the tour traversal is repeated.

\section{Extension for partitioned areas}
\label{sec:partition}

The strategies presented in Section~\ref{sec:convex} work best when the area is convex and free of obstacles. In arbitrary shaped environments the concurrent tree traversal (TT) approach from \cite{Mosteo2009} can be used. In this approach a Steiner tree is generated that contains all sensing locations. The robots start at the root (base station) and traverse the edges placing relays when necessary. Branches are traversed concurrently if possible, i.e. robots split into groups at branching points. After the leaves of a branch have been visited, the robots retreat to the branching point and wait for other groups traversing other branches to recombine the groups if a split happened. It is shown in \cite{Mosteo2009}, that the problem of determining how to split to explore the tree as fast as possible is NP-hard, and different heuristic strategies for determining when (\textit{splitting strategy}) and how (\textit{selection strategy}) to split are discussed. The \textit{early split} strategy splits the group at a branching point as soon as all leaves in subsequent branches, that have not been traversed yet, can be reached with the available number of robots. The \textit{late split} strategy splits the robots when all leaves in the subsequent branches, that have not been traversed yet, can be visited concurrently. When no split happens, the next branch has to be selected at a splitting point. The \textit{far selection} strategy selects the branch with the farthest leaf, and the \textit{near selection} strategy selects the branch with the nearest leaf. In total there are four different combinations of strategies to traverse a tree.

The direct application of the tree traversal approach requires a tree that contains all sensing location. To combine the advantages of the tree traversal with convex area coverage we consider a partitioning of the area into convex partitions where the areas are the vertices of the tree, and we assume that a partitioning of the area is given (e.g. determined by some algorithm described in \cite{LaValle2006}, \cite{ORourke1987}). For each partition a cell is defined which is the release point for the robots covering that partition. The release points are connected with edges that form a tree (to be determined after partitioning), and two numbers are associated with each edge: how many relays are necessary to build a chain between the release points (which is the number of cells on the shortest path in $G_C$ and an optimal solution of CMRD on a grid), and how long it takes to travel between the release points (which is the number of cells on the shortest path in $G_M$).

The problem in \cite{Mosteo2009} is different from the problem discussed here in two aspects. First, we consider partitions and not only single vertices. Each release point is considered as own branch to which a number of robots has to be assigned for covering the partition. To cover a partition, a certain minimum number of robots is necessary due to the limited communication range. Second, since we consider persistent surveillance, the groups do not have to retreat to the base station and recombine there if all partitions up to a certain depth in the tree are already covered concurrently. 

Our strategy works as follows. The robots start at the base station and move along an edge to a release point, dropping relays when necessary. When the robots reach a release point, one relay stays at the release point and a the remaining robots are split into subgroups according to the same rules described above. The robots assigned to a partition start to cover the area according to one of the algorithms described in Section~\ref{sec:convex} and return to the release point after all sensing locations have been visited. Here, they cover the area again, wait and recombine with other groups to move to the next release point, or retreat to a lower release point in the tree. We formally define the problem as:

\begin{definition}[d-CMPSTT]
The problem \mbox{d-CMPSTT} is defined by a set of tuples of the form $(A, B, \Gamma, \Delta, b, r, T)$. We denote $V_P$ the set of partitions, $A:V_P \rightarrow \mathbb{N}_{\geq 0}$ is the minimum number of robots necessary to cover a partition, $B:(V_P \times V_P) \rightarrow \mathbb{N}_{\geq 0}$ is the number of relays necessary between two release points, $\Gamma: (V_P \times \mathbb{N}_{>0}) \rightarrow \mathbb{N}_{\geq 0}$ is the time it takes to cover a partition with a certain number of robots, $\Delta: (V_P \times V_P) \rightarrow \mathbb{N}_{\geq 0}$ is the travel time between two release points. The vertex $b\in V_P$ is the partition containing the base station. The tree is given by the functions $B$ and $\Delta$. The number of robots is $r$, and $T$ is a time bound for $WI$.
\end{definition}

A solution of a $CMPSTT$ instance is the information about how the robots should split at each release point, which includes the branches visited concurrently, the order of the branches in the case not all can be visited concurrently, and how many robots are assigned to each branch. An example is shown in Figure~\ref{fig:example_cmpstt}. Assume $A(v_i)=2, A(b)=0, B\equiv0, \Gamma(v_i)=0, \Gamma(b)=0, \Delta\equiv1$. With $r=6$ a possible solution would be a splitting $((v_1,v_4))$ at the base station with the assigned number of robots $((4,2))$. This means, that branch $v_1$ and $v_4$ get visited concurrently and that 4 robots are assigned to branch $v_1$, and 2 robots are assigned to branch $v_4$. Since there are no more branches at the base station, $(v_1,v_4)$ is the only tuple in this splitting. A splitting in $v_1$ could be $((v_1, v_2),(v_1, v_3))$, where branches $v_1$ and $v_2$ are visited concurrently before the robots recombine at $v_1$ and then visit $v_1$ and $v_3$ concurrently. After recombination at $v_1$, this cycle starts again. The number of robots assigned to the branches are $((2,2),(2,2))$. In this solution the robots assigned to $v_1$ and $v_4$ at the base station do not have to retreat and combine at the base station, since the branches $v_1$ and $v_4$ are visited concurrently and there are no sensing locations in the base station partition, which results in $WI=6$ (which is  the idleness for $v_2$ and $v_3$). In this notation of a splitting it is allowed that a branch occurs multiple times in a splitting (e.g. at splitting point $v_1$). To be able to enumerate a splitting explicitly and prevent an infinite sequence of tuples in a splitting, we restrict the occurrences of branches in a splitting to the number of branches at the splitting point.

\begin{figure}[t]
	\centering
	\includegraphics[scale=0.6]{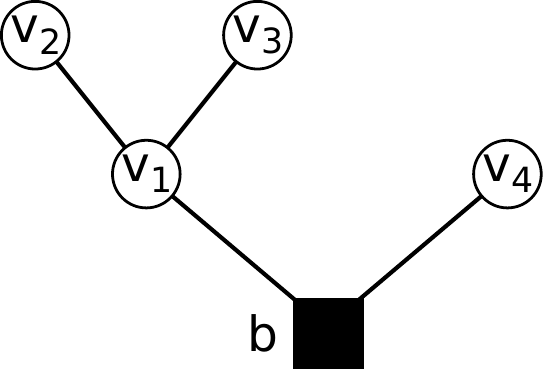}
	\caption{Example of a tree for CMPSTT.}
	\label{fig:example_cmpstt}
\end{figure}

In \cite{Mosteo2009} it is shown by a reduction from multi-processor scheduling that the problem of determining the order in which the branches should be visited is NP-hard. Since this is a special case of CMPSTT (with $A(p)=1, \forall p \in V_P$), CMPSTT it is NP-hard as well. Nevertheless, we can show that CMPSTT with $A(p) > 1$ for some $p$ is also NP-hard. We show this with a reduction from the number partition problem\footnote{The number partition problem is defined as the problem of finding a partition of a set $S=\{s_1, \ldots, s_\alpha\}$, $s_i \in \mathbb{N}_{>0}$, into two subsets $S_1$ and $S_2$ such that the sum of the elements in $S_1$ is equal to the sum of the elements in $S_2$.}. The details of the reduction are shown in the proof of

\begin{proposition}
d-CMPSTT is NP-hard.
\end{proposition}

\begin{proof}
An instance of a number partition problem can be transformed into an instance of \mbox{d-CMPSTT} (the edges $E_A$ of the tree are enumerated explicitly for convenience)
\begin{itemize}
	\item $V_P=\{b, v_1, \ldots, v_\alpha\}$,
	\item $E_A=\{[b, v_i]: i=1, \ldots, \alpha\}$
	\item $A(b)=0, A(v_i)=s_i, i=1, \ldots, \alpha$
	\item $B(b, v_i)=0, i=1, \ldots, \alpha$
	\item $\Gamma(b, \cdot)=0, \Gamma(v_i, \cdot)=1, i=1, \ldots, \alpha$
	\item $\Delta(b, v_i)=1, i=1, \ldots, \alpha$
	\item $r=1/2 \sum_{i=1, \ldots, \alpha}{s_i}$
	\item $T=6$
\end{itemize}
A solution to the number partition problem gives a solution to \mbox{d-CMPSTT} with $WI=6$: Let the subsets be ordered such that $\sum_{i=1, \ldots, k}{s_i}=r=1/2 \sum_{i=1, \ldots, \alpha}{s_i}$ with $k<\alpha$. Then, the robots visit the partitions $v_1, \ldots v_k$ concurrently starting at the base station, retreat to the base station and visit the partitions $v_{k+1}, \ldots, v_\alpha$. After the robots retreat again, this cycle starts from the beginning. The number of robots assigned to each partition is determined by $A(v_i)=s_i, i=1, \ldots, \alpha$.

Assume that there is a solution to \mbox{d-CMPSTT} that results in $WI=6$, i.e. each partition gets visited once or twice in each cycle since $r<\sum_{i=1, \ldots, \alpha}{A(v_i)}$, and the idleness of each partition is smaller or equal to 6. Assume also that this solution cannot be converted to a solution to the number partition problem. Then the following two situations can occur: (i) $WI=3$ for some partitions, i.e. it gets visited twice in each cycle, or (ii) the number of robots assigned to a branch $v_i$ is larger than $A(v_i)=s_i$.

In the first case $v_1, \ldots v_k$ are the partitions that get visited only in the first half of the cycle, $v_{k+1}, \ldots v_l$ are the partitions that get visited only in the second half of the cycle, and $v_{l+1}, \ldots v_\alpha$ are the partitions that get visited twice in each cycle. It follows that
\begin{align}
\sum_{i=1}^{k}{A(v_i)}+\sum_{i=l+1}^{\alpha}{A(v_i)} &\leq r, \text{ and} \nonumber \\
\sum_{i=k+1}^{l}{A(v_i)}+\sum_{i=l+1}^{\alpha}{A(v_i)} &\leq r, \text{ and therefore} \nonumber \\
\frac{1}{2}\sum_{i=1}^{l}{A(v_i)}+\sum_{i=l+1}^{\alpha}{A(v_i)} &\leq r, \nonumber
\end{align}
which can only be the case if the second sum in the last equation is zero, which is a contradiction to the definition of the number partition problem, or there are no partitions which get visited twice, which is a contradiction to the assumption.

In the second case, a similar argument holds:
\begin{align}
\sum_{i=1}^{k}{A(v_i)}+r_1 &\leq r, \text{ and} \nonumber \\
\sum_{i=k+1}^{l}{A(v_i)}+r_2 &\leq r, \text{ and therefore} \nonumber \\
\frac{1}{2}\sum_{i=1}^{l}{A(v_i)}+\frac{r_1+r_2}{2} &\leq r, \nonumber
\end{align}
where $r_1$ and $r_2$ are the surplus robots assigned to the branches. According to the assumption, more robots are assigned to the branches than necessary, and these robots are the surplus robots. From the last equation follows that $r_1$ and $r_2$ must be zero, which is a contradiction to the assumption. Therefore, a splitting for the d-CMPSTT instance has the form $((v_1, \ldots, v_k), (v_{k+1}, \ldots, v_\alpha))$, which constitutes a solution for the number partition problem.
\end{proof}

\section{Simulation results}
\label{sec:eval}

A first simulation is conducted on a grid of 30$\times$30 cells with the base station at the lower left corner and four sensing locations: $s_1$ at cell coordinates $(30,1)$, $s_2$ at $(30,30)$, $s_3$ at $(1,30)$ and $s_4$ at $(15,15)$. The communication range $R^{com}$ is half of the diagonal diameter of the grid, and the number of robots $r=3$. The results are shown in Figure~\ref{fig:time_simple}. In this simple scenario SHC shows the optimal behavior resulting in the best possible $WI$ of 60 (twice the side length of the area): One robot acts as relay for the other two and visits $s_4$, while another robot is commuting between $s_1$ and $s_2$ and the third robot between $s_2$ and $s_3$. SH and FH visit one sensing location after the other. Because the robots start at the base station, the instantaneous worst idleness $WI_t$ grows until all sensing locations have been visited for the first time before dropping to $WI$. The optimal behavior of SHC does not sustain on larger scenarios as can be seen in subsequent simulations.

We conduct the remaining simulation studies on a grid of 30$\times$30 cells with 900 sensing locations and with the base station at the lower left corner. The parameters $\omega_0$ and $\omega_1$ are determined with the pattern search algorithm provided by the Global Optimization Toolbox from Matlab. This algorithm systematically samples the parameter space and converges to a local minimum \cite{Audet2002}. For each pair $(\omega_0, \omega_1)$ polled by the pattern search algorithm, SH (Algorithm~\ref{alg:sh}) or SHC (Algorithm~\ref{alg:shc}) is executed for a predefined number of time steps $T_O$ (optimization horizon), and the objective for the resulting solution $\pi=(p_0, \ldots, p_{T_O})$ is evaluated. The objective value of a solution $\pi$ is the smallest time step when all sensing locations have been visited at least once or $T_O$ if not all sensing locations have been visited within the optimization horizon. This objective value is denoted as coverage time $CT$ and formally defined as $CT:=\min\{\argmin_{0 \leq \tau \leq T_O}\{S_{\tau}:V_S\subseteq S_{\tau}\} \cup \{T_O\}\}$, with $S_{\tau}:=\bigcup_{0\leq t \leq \tau} \{p_t(i): i\in R\}$. To reduce the optimization time, SH and SHC are aborted in each run in the course of the optimization as soon as all sensing locations have been visited at least once. We set the optimization horizon $T_O$ to 1800 time steps.

Pattern search uses a mesh grid to poll points around a current center in each iteration\cite{MathWorksPolling}. We set the start point of the search to $x_0=(0, 0)$, and pattern search polls the points $x_0, x_0+(1,0)$, $x_0+(0,1)$, $x_0+(-1,0)$, $x_0+(0,-1)$ in the first iteration with an initial mesh of size $1$. The point $x_1$ with the lowest objective value is the center of the new mesh in the second iteration. If $x_i$ is different from $x_{i-1}$, the mesh size is multiplied by $2$, and the points $x_1+2(1, 0), x_1+2(0, 1), x_1+2(-1, 0), x_1+2(0, -1)$ are polled in the second iteration. If none of the polled points has a better objective value than $x_i$, $x_i$ becomes the center of a new mesh, which has the size of the old mesh multiplied by $0.5$. This procedure continues until a user defined stopping criterion is met. We set the stopping criterion to a mesh tolerance of $0.01$. This procedure can be parallelized, and we use 4 cores to evaluate polls around a mesh center in parallel.

\begin{figure}[t]
	\center
	\begin{tikzpicture}
		\begin{axis}[
			height=5cm,
			width=\columnwidth,
			legend style={at={(0, 1.05)}, anchor=north west, legend columns=2, legend cell align=left, font=\footnotesize}, 
			ylabel=$WI_t$,
			xlabel=Time steps,
			xmin=0,
			ymin=0,
			ymax=120,
			axis lines*=left,
			]
		
			\addplot
			[style=solid]
			table[comment chars={C}, x index=0, y index=1]
			{data/simple/hist_sh_l.dat};
			
			\addplot
			[style=densely dotted]
			table[comment chars={C}, x index=0, y index=1]
			{data/simple/hist_shc_l.dat};
			
			
			\legend{SH/FH\\SHC\\}
		
		\end{axis}
	\end{tikzpicture}
	\caption{Comparison of the instantaneous idleness $WI_t$ over time for a simple example with 4 sensing locations and 3 robots.}
	\label{fig:time_simple}
\end{figure}
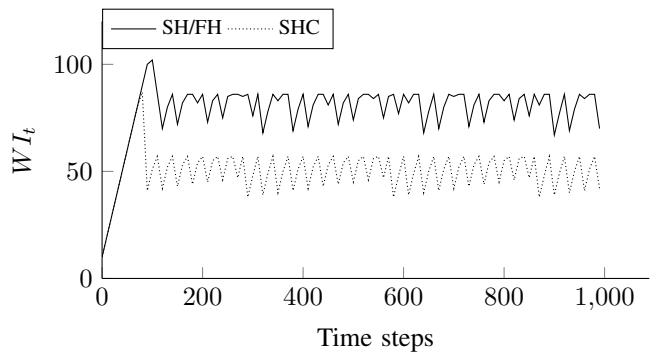

In the following experiment we evaluate the performance of SH, SHC, FH and TT with an increasing number of robots $r$ and a decreasing communication range $R^{com}$. The communication range has been chosen such that all $r$ robots are necessary to reach the upper right sensing location in the area. Figure~\ref{fig:first_urcom_l} shows the coverage time $CT$, which is the objective value of the parameter optimization for SH and SHC as described above. $CT$ can serve as an estimate of $WI$ for comparing the different algorithms. Figure~\ref{fig:first_urcom_l} shows that SH benefits from an increasing number of robots up to $r=10$ but suffers from the low communication range with $r=15$ and is not able to visit all sensing location within the time horizon. For the TT algorithm a tree has been chosen that is the union of all the shortest paths from each sensing location to the base station. All four combinations of strategies described in Section~\ref{sec:partition} are tried and the best value is recorded. This tree is traversed basically in a depth-first order without any concurrency, resulting in a $CT$ that is twice the number of sensing locations because in a tree traversal each cell in the tree is visited at least twice.

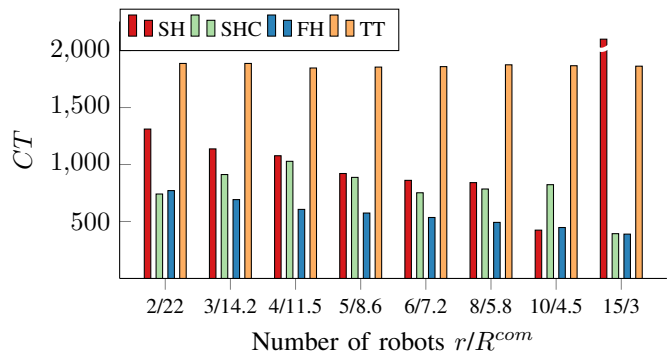
\begin{figure}[t]
\center
\begin{tikzpicture}
	\begin{axis}[
		height=5cm,
		width=\columnwidth,
		ybar,
		enlarge x limits,
		bar width=0.1,
		legend style={at={(0.0, 1.05)}, anchor=north west, legend columns=4, legend cell align=left, font=\footnotesize}, 
		xlabel=Number of robots $r$/$R^{com}$,
		xtick=data,
		xticklabels={2/22, 3/14.2, 4/11.5, 5/8.6, 6/7.2, 8/5.8, 10/4.5, 15/3},	
		xticklabel style={font=\footnotesize},
		ylabel=$CT$, 
		ymin=0,
		ymax=2050,
		ytick={500, 1000, 1500, 2000},
		enlarge y limits=upper,
		restrict y to domain*=0:2100, 
		after end axis/.code={ 
			\draw [ultra thick, white, decoration={snake, amplitude=1pt}, decorate] (rel axis cs:0.02,0.9) -- (rel axis cs:1,0.9);
			},
		axis lines*=left,
		clip=false,
		]

		\addplot[draw=black, fill=C1]
		table[header=false, x expr=\coordindex+1, y index=4]
			{data/first_urcom/res_sh_l.dat};
		
		\addplot[draw=black, fill=C2]
		table[header=false, x expr=\coordindex+1, y index=4]
			{data/first_urcom/res_shc_l.dat};
		
		\addplot[draw=black, fill=C3]
		table[header=false, x expr=\coordindex+1, y index=4]
			{data/first_urcom/res_tsp_l.dat};
	
		\addplot[draw=black, fill=C4]
		table[header=false, x expr=\coordindex+1, y index=4]
			{data/first_urcom/res_tt_l.dat};
	
		\legend{SH\\SHC\\FH\\TT\\}

	\end{axis}
\end{tikzpicture}
\caption{Comparison of $CT$ with increasing number of robots $r$ and decreasing communication range $R^{com}$. The broken bar for SH indicates that it was not able to visit all sensing locations within the provided time horizon for the optimization.}
\label{fig:first_urcom_l}
\end{figure}

In the next experiment we set the communication range to a constant value of a quarter of the diameter of the area, such that 4 robots are necessary to reach all sensing locations. The results are shown in Figure~\ref{fig:first_u}. It shows that SH benefits most from an increasing number of robots and $CT$ approaches the lower bound $|V_S|/r$. The reason for the drop of TT at $r=8$ is the splitting and selection strategy. There are three branches originating at the base station, one which covers mainly the lower right, one the diagonal, and one the upper right proportion of the area. With 6 and 8 robots all branches are visited subsequently, which results in a lower $CT$ for 8 robots. With 10 robots the diagonal (which contains a small number of sensing locations) and lower right (which contains a larger number of sensing locations) branches are visited concurrently, which results in the situation that the robots visiting the diagonal proportion finish earlier and stay idle while waiting for recombination of the split group.

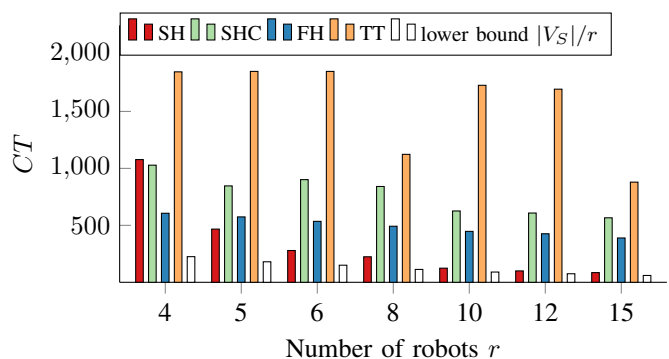
\begin{figure}[t]
\center
\begin{tikzpicture}
	\begin{axis}[
		height=5cm,
		width=\columnwidth,
		ybar,
		enlarge x limits,
		bar width=0.1,
		legend style={at={(0.0, 1.05)}, anchor=north west, legend columns=5, legend cell align=left, font=\footnotesize}, 
		xlabel=Number of robots $r$,
		xtick=data,
		xticklabels={4,5,6,8,10,12,15},
		ylabel=$CT$, 
		ymin=0,
		ymax=2050,
		ytick={500, 1000, 1500, 2000},
		enlarge y limits=upper,
		restrict y to domain*=0:2100, 
		after end axis/.code={ 
			\draw [ultra thick, white, decoration={snake, amplitude=1pt}, decorate] (rel axis cs:0.02,0.9) -- (rel axis cs:1,0.9);
			},
		axis lines*=left,
		clip=false,
		]

		\addplot[draw=black, fill=C1]
		table[header=false, x expr=\coordindex+1, y index=4]
			{data/first_u/res_sh_l.dat};
		
		\addplot[draw=black, fill=C2]
		table[header=false, x expr=\coordindex+1, y index=4]
			{data/first_u/res_shc_l.dat};
		
		\addplot[draw=black, fill=C3]
		table[header=false, x expr=\coordindex+1, y index=4]
			{data/first_u/res_tsp_l.dat};
	
		\addplot[draw=black, fill=C4]
		table[header=false, x expr=\coordindex+1, y index=4]
			{data/first_u/res_tt_l.dat};
	
		\addplot[draw=black, fill=white]
		table[header=false, x expr=\coordindex+1, y index=4]
			{data/first_u/res_lb_l.dat};

		\legend{SH\\SHC\\FH\\TT\\lower bound $|V_S|/r$\\}

	\end{axis}
\end{tikzpicture}
\caption{Comparison of $CT$ with increasing number of robots $r$ and fixed communication range $R^{com}$ (a quarter of the diameter of the area).}
\label{fig:first_u}
\end{figure}

Figure~\ref{fig:first_urcom_cr} shows $CT$ for an increasing number of robots $r$ and a decreasing communication range $R^{com}$, where a number of robots are used as relays with fixed positions between the base station and a release point near the center of the area. The scenario for $r=5$ ($3$ sensing robots) is shown in Figure~\ref{fig:scenario_first_urcom_cr}. The remaining robots start to cover the area from the release point. The purpose is to assess whether the release point for a partitioning should be at the center or at the corner closest to the base station of a partition. The number of robots in total and the communication range is the same as in Figure~\ref{fig:first_urcom_l}. The numbers on the x-axis in Figure~\ref{fig:first_urcom_cr} are the numbers of robots that are available for coverage of the area. Having a longer chain of relays to shift the release point closer to the center results in a decreased number of robots available for visiting sensing locations and the expected behavior with a worse performance can be observed.

\begin{figure}[t]
	\centering
	\begin{tabular}{ccc}
		\subfloat[]{
			\includegraphics[scale=0.34]{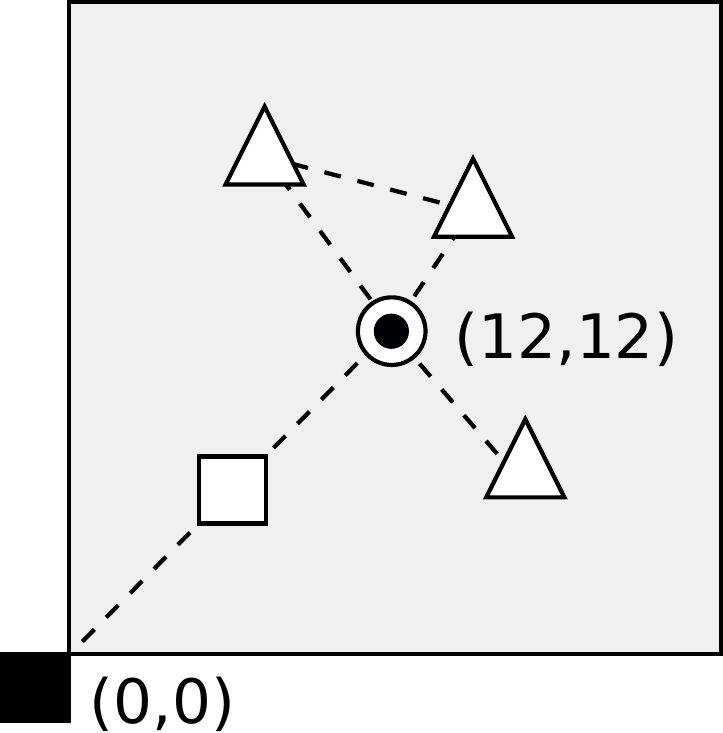}
			\label{fig:scenario_first_urcom_cr}
		}
		&
		\subfloat[]{
			\includegraphics[scale=0.34]{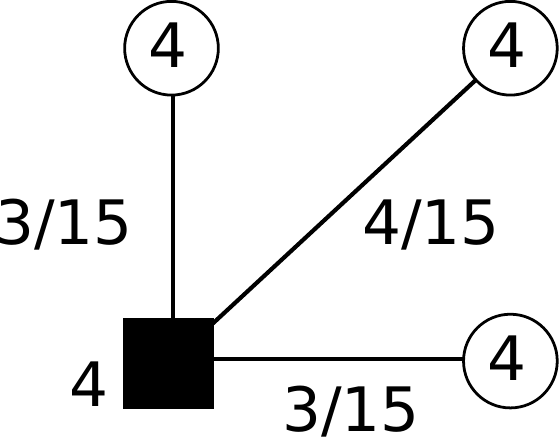}
			\label{fig:scenario_pers_u_att_2}
		}
		&
		\subfloat[]{
			\includegraphics[scale=0.34]{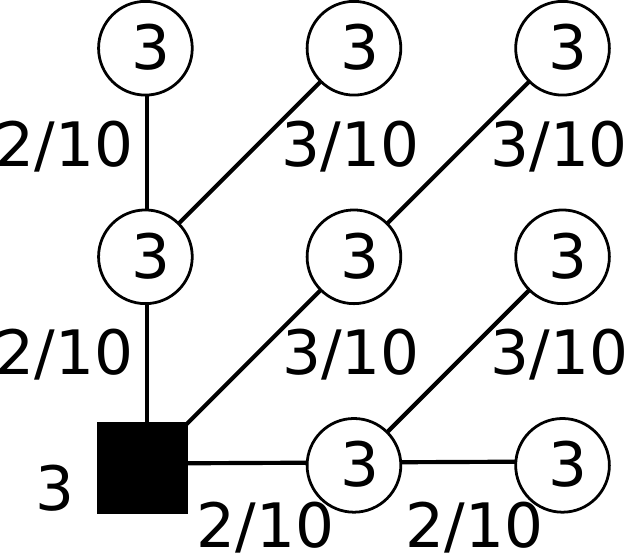}
			\label{fig:scenario_pers_u_att_3}
		}
	\end{tabular}
	\caption{(a) Illustration of the scenario with $5$ robots ($3$ sensing robots) for the results in Figure~\ref{fig:first_urcom_cr} (see Figure~\ref{fig:scenario_cmpstt} for the meaning of the symbols). The numbers show the x and y-coordinates of the cells of the base station and the release point. (b) Tree for CMPSTT for the results in Figure~\ref{fig:pers_u_att_2}, and (c) Figure~\ref{fig:pers_u_att_3}, respectively. The numbers in the vertices and next to the base station (the base station is release point of the lower left partition) define $A$, and the numbers next to the edges define $B/\Delta$ of the CMPSTT instances.}
\label{fig:scenarios_simulation}
\end{figure}

\begin{figure}[t]
\center
\begin{tikzpicture}
	\begin{axis}[
		height=5cm,
		width=\columnwidth,
		ybar,
		enlarge x limits,
		bar width=0.1,
		legend style={at={(0.0, 1.05)}, anchor=north west, legend columns=4, legend cell align=left, font=\footnotesize}, 
		xlabel=Number of robots $r$,
		xtick=data,
		xticklabels={1,2,2,3,3,4,5,8},
		ylabel=$CT$, 
		ymin=0,
		ymax=2050,
		ytick={500, 1000, 1500, 2000},
		enlarge y limits=upper,
		restrict y to domain*=0:2100, 
		after end axis/.code={ 
			\draw [ultra thick, white, decoration={snake, amplitude=1pt}, decorate] (rel axis cs:0.02,0.8) -- (rel axis cs:1,0.9);
			},
		axis lines*=left,
		clip=false,
		]

		\addplot[draw=black, fill=C1]
		table[header=false, x expr=\coordindex+1, y index=4]
			{data/first_urcom/res_sh_cr.dat};
		
		\addplot[draw=black, fill=C2]
		table[header=false, x expr=\coordindex+1, y index=4]
			{data/first_urcom/res_shc_cr.dat};
		
		\addplot[draw=black, fill=C3]
		table[header=false, x expr=\coordindex+1, y index=4]
			{data/first_urcom/res_tsp_cr.dat};
	
		\addplot[draw=black, fill=C4]
		table[header=false, x expr=\coordindex+1, y index=4]
			{data/first_urcom/res_tt_cr.dat};
	
		\legend{SH\\SHC\\FH\\TT\\}

	\end{axis}
\end{tikzpicture}
\caption{Comparison of $CT$ with increasing number of robots $r$ and decreasing communication range $R^{com}$ with a certain number of relays between the base station and a release point near the center of the area. The number of robots in total and the communication range is the same as in Figure~\ref{fig:first_urcom_l} for each scenario. The numbers on the x-axis are the numbers of robots that are available for coverage of the area.}
\label{fig:first_urcom_cr}
\end{figure}
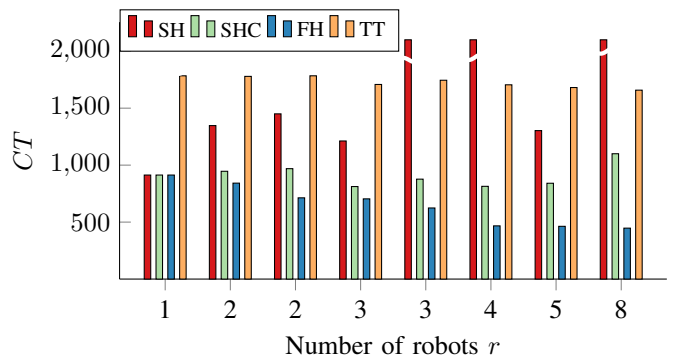

Next, we investigate in $WI$ for area partitioning. We do not consider an non-convex area with obstacles explicitly, but a regular partitioning of the convex area into convex subareas. This allows us to compare the performance on the unpartitioned case with the performance on the partitioned case, and we do not rely on an algorithm for dividing a non-convex polygon into a set of convex polygons. For non-cyclic algorithms $WI$ can only be estimated over a limited time horizon. Figure~\ref{fig:pers_u} shows the results for the unpartitioned area for a time horizon of 3000 time steps. $WI$ is calculated as the maximum time between two consecutive visits within the time window after the first and before the last visit of a sensing location in the horizon and over all sensing locations. For comparison with the tree traversal of partitions (CMPSTT), the area is partitioned into $2\times2$ (Figure~\ref{fig:pers_u_att_2}) and $3\times3$ (Figure~\ref{fig:pers_u_att_3}) rectangular partitions of equal size and traversed with different algorithms (SH, SHC, FH) for convex partitions. The release point for each partition is the lower left corner of the partition (the base station is the release point of the lower left partition). The trees for CMPSTT are shown in Figure~\ref{fig:scenario_pers_u_att_2} and Figure~\ref{fig:scenario_pers_u_att_3}. The parameters for SH and SHC have been determined for the whole area for the case of an unpartitioned area and for a single partition in the case of a partitioning. The optimization times for determining $(\omega_0, \omega_1)$ have been recorded for each scenario and are shown in Figure~\ref{fig:optim_time_u} for SH. For the partitioning, the optimization times for the different number of robots assigned to a partition are summed up, e.g. if the different number of robots assigned to any partition are $r_1, \ldots, r_k$, then the optimization times for $r_1, \ldots, r_k$ robots are summed up (since all these different parameters $(\omega_0, \omega_1)$ for different number of robots are required).

Although the performance is comparable between the unpartitioned and partitioned scenario up to 20 robots for SH, optimizing for smaller areas can greatly reduce the optimization time. As the number of robots gets larger, the number of robots assigned to partitions get smaller because the partitions are covered concurrently, and therefore the optimization time decreases for an increasing total number of robots. The performance of SHC and FH can even be improved by partitioning the area. For these approaches using a smaller number of robots on a smaller area is more effective, because the performance improvement is smaller as compared to SH with an increasing number of robots (cp. Figure~\ref{fig:first_u} and  Figure~\ref{fig:pers_u}).

\begin{figure}[t]
\center
\begin{tikzpicture}
	\begin{axis}[
		height=5cm,
		width=\columnwidth,
		ybar,
		enlarge x limits,
		bar width=0.1,
		legend style={at={(0.0, 1.05)}, anchor=north west, legend columns=4, legend cell align=left, font=\footnotesize}, 
		xlabel=Number of robots $r$,
		xtick=data,
		xticklabels={8,10,12,14,16,18,20,22,24,26},
		ylabel=$WI$,
		ymin=0,
		ymax=1500,
		ytick={500, 1000, 1500, 2000},
		axis lines*=left,
		clip=false,
		]

		\addplot[draw=black, fill=C1, restrict x to domain=1:10]
		table[header=false, x expr=\coordindex+1, y index=2]
			{data/persistent/res_sh_l.dat};
		
		\addplot[draw=black, fill=C2, restrict x to domain=1:10]
		table[header=false, x expr=\coordindex+1, y index=2]
			{data/persistent/res_shc_l.dat};
		
		\addplot[draw=black, fill=C3, restrict x to domain=1:10]
		table[header=false, x expr=\coordindex+1, y index=2]
			{data/persistent/res_tsp_l.dat};
	
		\legend{SH\\SHC\\FH\\}

	\end{axis}
\end{tikzpicture}
\caption{Comparison of $WI$ with an increasing number of robots $r$ and constant communication range $R^{com}=5.8$ over a time horizon of 3000 time steps.}
\label{fig:pers_u}
\end{figure}
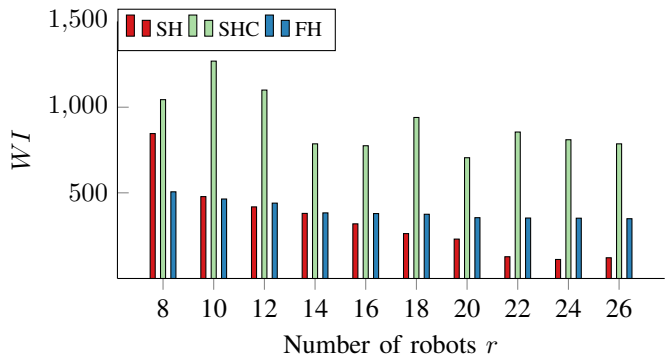

\begin{figure}[t]
	\center
	\begin{tikzpicture}
	\begin{axis}[
	height=5cm,
	width=\columnwidth,
	ybar,
	enlarge x limits,
	bar width=0.1,
	legend style={at={(0.0, 1.05)}, anchor=north west, legend columns=3, legend cell align=left, font=\footnotesize}, 
	xlabel=Number of robots $r$,
	xtick=data,
	xticklabels={8,10,12,14,16,18,20,22,24,26},
	ylabel=$WI$,
	ymin=0,
	ymax=1500,
	ytick={500, 1000, 1500, 2000},
	axis lines*=left,
	clip=false,
	]
	
	\addplot[draw=black, fill=C1, restrict x to domain=1:10]
	table[header=false, x expr=\coordindex+1, y index=2]
	{data/persistent/res_sh_att_2_l.dat};
	
	\addplot[draw=black, fill=C2, restrict x to domain=1:10]
	table[header=false, x expr=\coordindex+1, y index=2]
	{data/persistent/res_shc_att_2_l.dat};
	
	\addplot[draw=black, fill=C3, restrict x to domain=1:10]
	table[header=false, x expr=\coordindex+1, y index=2]
	{data/persistent/res_tsp_att_2_l.dat};
	
	\legend{CMPSTT (SH)\\CMPSTT (SHC)\\CMPSTT (FH)\\}
	
	\end{axis}
	\end{tikzpicture}
	\caption{Comparison of $WI$ with an increasing number of robots $r$ and constant communication range $R^{com}=5.8$ for the tree traversal of partitions (CMPSTT) of the area partitioned into $2\times2$ equally sized rectangular partitions and for different algorithms for the convex partitions (SH, SHC, FH).}
	\label{fig:pers_u_att_2}
\end{figure}
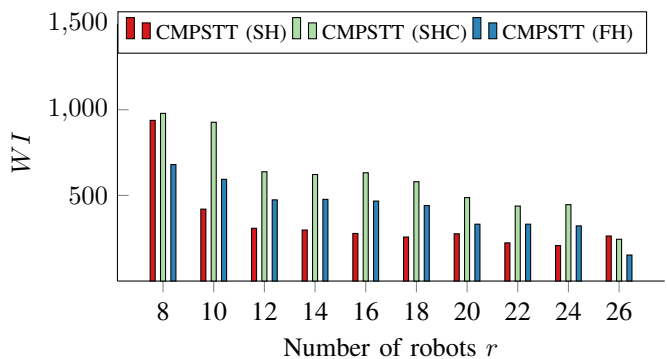

\begin{figure}[t]
	\center
	\begin{tikzpicture}
	\begin{axis}[
	height=5cm,
	width=\columnwidth,
	ybar,
	enlarge x limits,
	bar width=0.1,
	legend style={at={(0.0, 1.05)}, anchor=north west, legend columns=3, legend cell align=left, font=\footnotesize}, 
	xlabel=Number of robots $r$,
	xtick=data,
	xticklabels={8,10,12,14,16,18,20,22,24,26,28},
	ylabel=$WI$,
	ymin=0,
	ymax=1500,
	ytick={500, 1000, 1500, 2000},
	axis lines*=left,
	clip=false,
	]
	
	\addplot[draw=black, fill=C1, restrict x to domain=1:10]
	table[header=false, x expr=\coordindex+1, y index=2]
	{data/persistent/res_sh_att_3_l.dat};
	
	\addplot[draw=black, fill=C2, restrict x to domain=1:10]
	table[header=false, x expr=\coordindex+1, y index=2]
	{data/persistent/res_shc_att_3_l.dat};
	
	\addplot[draw=black, fill=C3, restrict x to domain=1:10]
	table[header=false, x expr=\coordindex+1, y index=2]
	{data/persistent/res_tsp_att_3_l.dat};
	
	\legend{CMPSTT (SH)\\CMPSTT (SHC)\\CMPSTT (FH)\\}
	
	\end{axis}
	\end{tikzpicture}
	\caption{Comparison of $WI$ with an increasing number of robots $r$ and constant communication range $R^{com}=5.8$ for the tree traversal of partitions (CMPSTT) of the area partitioned into $3\times3$ equally sized rectangular partitions and for different algorithms for the convex partitions (SH, SHC, FH). In this scenario it is not possible to reach all sensing locations with 8 robots.}
	\label{fig:pers_u_att_3}
\end{figure}
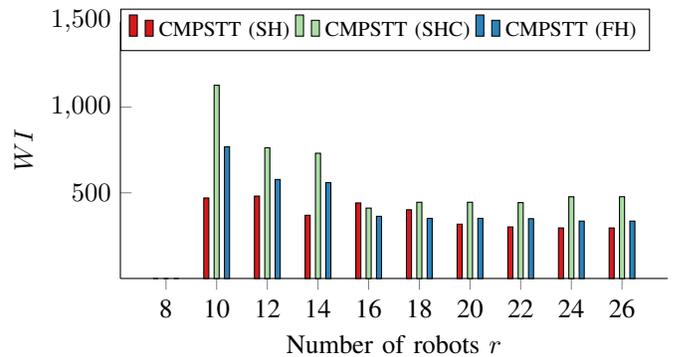

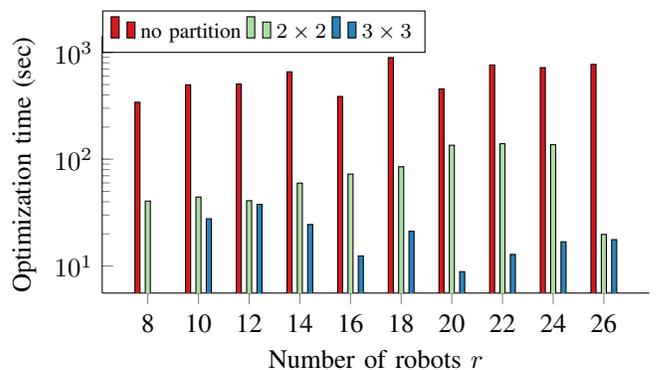
\begin{figure}[t]
\center
\begin{tikzpicture}
	\begin{semilogyaxis}[
		height=5cm,
		width=\columnwidth,
		ybar,
		enlarge x limits,
		bar width=0.1,
		legend style={at={(0.0, 1.1)}, anchor=north west, legend columns=4, legend cell align=left, font=\footnotesize}, 
		xlabel=Number of robots $r$,
		xtick=data,
		xticklabels={8,10,12,14,16,18,20,22,24,26},
		ylabel=Optimization time (sec),
		ytick={1e1, 1e2, 1e3, 1e4},
		axis lines*=left,
		]

		\addplot[draw=black, fill=C1, restrict x to domain=1:10]
		table[header=false, x expr=\coordindex+1, y index=5]
			{data/partparam/res_sh_l.dat};
		
		\addplot[draw=black, fill=C2, restrict x to domain=1:10]
		table[header=false, x expr=\coordindex+1, y index=1]
			{data/partparam/res_sh_att_time_2_l.dat};
		
		\addplot[draw=black, fill=C3, restrict x to domain=1:10]
		table[header=false, x expr=\coordindex+1, y index=1]
			{data/partparam/res_sh_att_time_3_l.dat};
	
		\legend{no partition\\$2\times2$\\$3\times3$\\}

	\end{semilogyaxis}
\end{tikzpicture}
\caption{Optimization time for SH for an increasing number of robots $r$ for the unpartitioned area and for partitions of different size.}
\label{fig:optim_time_u}
\end{figure}

\section{Conclusion}
\label{sec:conclusion}

In this work we have investigated path planning for multiple robots for persistent surveillance with connectivity constraints. We introduce the problem and related problem instances on graphs, which have not been considered yet in literature, and show that they are NP-hard. We propose several strategies on grids that can be applied to convex areas. On one hand, the simple short horizon (SH) strategy, which selects the next goal of a robot greedily without considering the anticipated goals of other robots, achieves the best performance when the number of robots is larger than the minimum number of robots required to reach all sensing locations. In this case the additional robots can be used more effectively than in the short horizon cooperative (SHC) and full horizon (FH) approaches. On the other hand, if all robots are required to reach all sensing locations, SHC and FH perform better. Although FH performs better than SHC in most cases, SHC does not require a preplanned tour through all sensing locations.

The short horizon approaches SH and SHC rely on parameters that have to be optimized before the mission execution. To apply these strategies on more general environments, which arise from discretization of real world scenarios, we propose a combination with a tree traversal approach. The simulation results on the considered scenarios indicate that the combination with tree traversal on a partitioned area does not impair the performance considerably up to a certain number of robots but reduces the optimization time substantially. With a larger number of robots the trade-off between performance and optimization time is revealed for SH. Partitioning can increase the performance for the other approaches, because using a smaller number of robots on smaller areas is more effective.

The presented algorithms can be used to compute paths that have to be followed by the robots, which requires a central entity for computation and synchronization among all robots at each time step. This approach has several limitations, which can be subject of future extensions. If a robot fails, the network gets disconnected and the remaining robots can continue the mission only after failure recovery measures. Furthermore, synchronization among all robots (at least among the robots within a partition) requires a reliable network connection that ensures a timely delivery of the state of each robot through the network to the central entity. Other possible extensions are algorithms for the presented problems that can be applied to general or other special types of graphs (e.g. where ${E_M\subseteq E_C}$).

\bibliographystyle{IEEEtran}
\bibliography{IEEEabrv,references}

%
%

\end{document}